%% file: arxiv.tex
\documentclass[sigconf, authorversion,nonacm]{acmart}

\AtBeginDocument{%
  }

\usepackage{algorithmic}
\usepackage{graphicx}
\usepackage{textcomp}
\usepackage{xcolor}
\usepackage{booktabs}
\usepackage{hyperref}
\usepackage{subfig}

\usepackage{enumitem}
\usepackage[normalem]{ulem}
\usepackage{tikz}
\usetikzlibrary{external}
\usepackage{pgfplots}
\usepackage{pgfplotstable}
\usepgfplotslibrary{colormaps}
\usepackage{etoolbox}

\usepackage{cleveref}
\crefname{lstlisting}{listing}{listings}
\Crefname{lstlisting}{Listing}{Listings}
\usepackage{listings}

\newtheorem{lemma}{Lemma}
\newtheorem{theorem}{Theorem}
\newtheorem{corollary}{Corollary}

\usepackage{algorithmic}
\usepackage{algorithm}

\newcommand{\rev}[2]{}

\begin{document}

\title{WCDT: Systematic \underline{WC}ET Optimization for\\ \underline{D}ecision \underline{T}ree Implementations}

\author{Nils Hölscher}
\email{nils.hoelscher@udo.edu}
\affiliation{%
  \institution{TU Dortmund University}
  \city{Dortmund}
  \country{Germany}
}

\author{Christian Hakert}
\email{christian.hakert@udo.edu}
\affiliation{%
  \institution{TU Dortmund University}
  \city{Dortmund}
  \country{Germany}
}

\author{Georg von der Brüggen}
\email{georg.von-der-brueggen@udo.edu}
\affiliation{%
  \institution{TU Dortmund University}
  \city{Dortmund}
  \country{Germany}
}

\author{Jian-Jia Chen}
\email{jian-jia.chen@udo.edu}
\affiliation{%
  \institution{TU Dortmund University}
  \city{Dortmund}
  \country{Germany}
}

\author{Kuan-Hsun Chen}
\email{k.h.chen@utwente.nl}
\affiliation{%
  \institution{University of Twente}
  \city{Twente}
  \country{Netherlands}
}

\author{Jan Reineke}
\email{reineke@cs.uni-saarland.de}
\affiliation{%
  \institution{Saarland University}
  \city{Saarbrücken}
  \country{Germany}
}
\begin{abstract}
Machine-learning models are increasingly deployed on resource-constrained embedded systems with strict timing constraints.
In such scenarios, the worst-case execution time (WCET) of the models
is required to ensure safe operation. Specifically, decision trees are a prominent class of machine-learning models and the main building blocks of tree-based ensemble models (e.g., random forests), which are commonly employed in resource-constrained embedded systems.

In this paper, we develop a systematic approach for WCET optimization of decision tree implementations.
To this end, we introduce a linear surrogate model that estimates the execution time of individual paths through a decision tree based on the path's length and the number of taken branches.
We provide an optimization algorithm that constructively builds
a WCET-optimal implementation of a given decision tree with respect to this surrogate model.
We experimentally evaluate both the surrogate model and the
WCET-optimization algorithm. The evaluation shows that the optimization algorithm improves analytically determined WCET by up to $17\%$ compared to an unoptimized implementation.
\end{abstract}

\maketitle

\pagestyle{plain}
\thispagestyle{plain}

\vspace*{-.2cm}

\section{Introduction}
\input{sections/intro.tex}

\section{Decision Tree Implementations}

\input{sections/forests.tex}
\section{WCET Surrogate Model}
\label{sec:surrogate}
\input{sections/algo.tex}
\section{Evaluation}
\label{sec:eval}

\input{sections/eval.tex}

\section{Related Work}
\input{sections/related.tex}
\label{sec:related}
\section{Conclusion and Outlook}
\input{sections/conclusion.tex}

\input{sections/outlook.tex}

\bibliographystyle{plain}
\bibliography{bib}

\end{document}

%% file: sections/intro.tex
Machine-learning (ML) models are not only growing larger, deeper, and more complex, but they are also increasingly employed in resource-constrained embedded systems.
 Tree-based models, such as random forests, excel in various domains due to their effectiveness with limited training data~\cite{10.5555/3172077.3172386} and superior interpretability~\cite{DBLP:journals/natmi/LundbergECDPNKH20}.
They have, for example, been utilized
in astrophysics~\cite{DBLP:conf/pkdd/BockermannBBEMR15}, real-time pedestrian detection~\cite{6751433}, real-time 3D face analysis~\cite{DBLP:journals/ijcv/FanelliDGFG13}, and remote sensing~\cite{sensing}.

While the average resource consumption of machine learning models is a well-explored area, there is still a notable gap in research focusing on enhancing worst-case execution time (WCET) guarantees.
In this paper, we address this gap by examining the WCET optimization
for the internal structure
of decision-tree implementations.
Optimizing the WCET of decision-tree implementations is highly relevant to enable resource-efficient systems with ML components.
Performance optimization on decision-tree implementations has been studied recently~\cite{8594826, chen2022efficient, 9925686,hakert2023immediate, 9923840}.
However, all these methods target the average-case execution time (ACET), and to the best of our knowledge no
results on WCET optimization of decision trees have been provided.

Applying a trained decision tree to new inputs, so-called model inference, amounts to following a path from the root of the tree to a leaf node, where the model output
is determined.
The ACET of model inference
depends on the input distribution, as this distribution determines
the probability for each possible paths through the tree, and optimizes the decision trees according to these distributions.
A decision tree is usually implemented in one
of two ways:\looseness=-1
\begin{itemize}[leftmargin=*]
\item A \emph{native-tree} implementation is a loop that iterates over the nodes of the decision tree, which are stored as a data structure in memory, starting from the root to a leaf node that contains the output of the decision, e.g., a classification result.
\item In an \emph{if-else-tree} implementation each node is directly encoded as an \texttt{if-else-statement}, where the subtrees are encoded in the \texttt{then} and \texttt{else} branches of the \texttt{if-else-statement}.
\end{itemize}

Existing average-case optimizations for decision trees, e.g.,~\cite{chen2022efficient, hakert2023immediate}, utilize the probability distribution of node accesses during the training phase as a proxy for the distribution at inference time.
These methods modify the memory layout of the decision trees to increase cache friendliness.
In particular, the tree nodes
are arranged to improve spatial locality, so that cache prefetching often loads the next accessed node, and eviction
targets nodes that are likely not used in the future.

For WCET optimization of decision trees, probability distributions of path accesses
are not relevant. Instead, the execution time of the path with the highest WCET, i.e., the worst-case execution path (WCEP), needs to be minimized.
Any change in the realization of the tree potentially not only affects the execution time of the WCEP, but also has the potential to make a different path the WCEP. As a consequence, to optimize a decision tree's WCET, its WCEP needs to be identified and optimized repeatedly.
Hence, conducting optimization of decision trees for their WCET crucially requires a model for the WCEP construction
and fast estimation of its WCET to omit time-consuming black-box WCET analysis during the optimization process.

\noindent\textbf{Our Contributions:}
In this work, we develop a WCET-optimization method for decision
trees that is centered around a fast and accurate surrogate model of path execution times:
\begin{itemize}[leftmargin=*]
\item We define a linear surrogate model in
  Section~\ref{sec:surrogate} that estimates the WCET of paths in a decision tree.
  We construct an accurate surrogate model based on extensive experiments with WCET analysis, which
  results in a one time overhead.
  We evaluate the derived model quality in two regards: ,
  We examine the expressiveness of the surrogate model
  and analyze to which degree it can model the
  collected data from the extensive WCET analysis.

  \item In Section~\ref{sec:optimization} we develop a greedy optimization algorithm that constructs an \emph{if-else-tree} realization of a given decision tree that provably minimizes the WCET w.r.t. the surrogate model.

\item Section~\ref{sec:eval} presents the evaluation results for a wide range of concrete decision tree models,
  trained on multiple datasets. We derive realizations of these trees
  with existing state-of-the-art ACET-optimization methods, as well as
  with our optimal model-centered WCET-optimizing algorithm. We
  compare the resulting estimated WCET for all realizations from the
  surrogate model as well as from precise WCET analysis, and
  show that the surrogate model serves as an accurate estimation of
  the analyzed WCET. As a consequence, the derivation of an optimal
  realization under the surrogate model delivers a highly WCET-optimized
  decision tree, which reduces the analytically determined WCET by up to $17\%$ compared to a straightforward realization.
\end{itemize}


%% file: sections/forests.tex
There is a wide range of techniques to train decision trees.
During training, for instance, several approaches for random sampling, bootstrapping, hyperparameter variations, and different split criteria can be considered.
Decision-tree training
is not the focus of this work, since the proposed optimization of WCET inference
is agnostic to how the logical decision tree structure is derived.
We instead focus on the efficient realization of the inference operation for given decision trees.
Consequently, in this section, we lay out the basic structure of binary decision trees and their relation to an implementation in C/C++, which is independent to the library used to create the trees\@.

A decision tree \(DT\) consists of \(\{n_0, n_1, \ldots, n_m\}\) a set of nodes, where \(n_0\) is the root node.
 One way to measure a decision tree's size is via its depth, the length of the longest path from the root to a leaf node.
We define the length of a path as the number of edges it contains, i.e., a decision tree consisting only of a root node has a depth of~0.

Every non-leaf node, also called inner node, is associated with a feature index \(FI(n_i)\), a split threshold \(SP(n_i)\), a left child \(LC(n_i)\), and a right child \(RC(n_i)\), whereas every leaf node is only associated with a prediction value \(P(n_j)\).
The inference operation of a decision tree \(\inf(DT, x)\) returns the prediction value of the corresponding leaf node \(P(n_t)\), which is determined by following a path of nodes through the tree \( n_{x_0}, n_{x_1}, \ldots, n_{x_y} \), where \(x_0\) is the root and \(x_y\) is a leaf node, according to the following inference rule:
\begin{equation}
    n_{x_{i+1}}=\left \{ \begin{array}{cc}
        LC(n_{x_i})&: x[FI(n_{x_i})] \leq SP(n_{x_i})\\
        RC(n_{x_i})&: x[FI(n_{x_i})] > SP(n_{x_i})
    \end{array}\right.
    \label{eq_childfollow}
\end{equation}

For the practical realization of the inference operation, two popular methods exist~\cite{Asadi/etal/2014}. One is the realization as so-called \emph{native trees},  where \Cref{eq_childfollow} is repeatedly evaluated in a loop and \(x_i\) is stored as a running index.
In this realization, tree nodes are stored as data objects in an array and \(x_i\) can be directly encoded as an index into this array.
The loop terminates whenever a leaf node is reached and the inference procedure returns the associated prediction value.
The potential for optimization lies in the layout of the data array to minimize the number of possible data cache misses on all paths through the tree.

The other common implementation approach is via so-called \emph{if-else-trees}, where the tree is unrolled into nested \texttt{if-else-statements}. The condition of \Cref{eq_childfollow} is checked for every node in a separate \texttt{if-statement}, the entire implementation of the left subtree is placed in the if-block, the implementation of the right subtree is placed in the else-block. Leaf nodes are translated into \texttt{return-statements} in this implementation, such that the execution is ended and the corresponding prediction value is returned.

For the rest of this paper, we focus on the implementation of if-else-trees,
since they have been shown to be faster than native trees in most cases by Buschj\"ager et~al.~\cite{8594826}.
They are also common in industrial practice; for instance, in the Treelite project (\url{https://treelite.readthedocs.io/en/latest/}), which is utilized by Amazon SageMaker Neo for inference at edge devices~\footnote{The original work was presented by Cho and Li~\cite{Cho2018}.}.
 We refer to the code directly behind a conditional branch instruction as the \emph{untaken slot} and to the code at the branch target location as the \emph{taken slot}. The compiler always translates an \texttt{if-statement} to a corresponding pair of comparison and conditional branch instruction, which effectively places either the if-block or the else-block in the taken or untaken slots of the branch instruction.
 In this work, we assume that compiler optimizations are disabled, since static WCET analysis tends to get more complex and possibly pessimistic on heavily compiler-optimized results~\cite{WCETO3}. Therefore, the same translation from source code to assembly code is always applied.

\section{Decision Tree Tuning and WCET Model}
\label{sec:tuning}

\noindent\textbf{Decision Tree Tuning.}
For \emph{if-else trees},  
implementation optimization can only
utilize a limited set of actions to transform the code structure into
a more hardware friendly, and therefore optimized, version, while
keeping the model logic and output untouched. In this paper,
we focus
only on \emph{flipping of the comparison in a decision tree
  node}. That is,
  of the two children nodes, one is modeled as a
{\emph{taken slot}} (i.e., the child of the \emph{else}
statement) and
the other is modeled as an {\emph{untaken
    slot}} (i.e., the child of the \emph{if} statement).  By
exchanging the entire untaken and taken slots (i.e. flipping the
subtree locations under the corresponding node) and 
inverting the branch condition, the logical correctness of the tree node
is maintained.

The untaken slots are placed directly adjacent to the branch condition and contiguously in memory.
When execution continues in the untaken slot, the following instructions are thus likely in the same cache block as the branch condition, which results in a reduced number of cache misses, and thus a lower WCET.

\noindent\textbf{WCET Analysis Tool and Microarchitecture Model.}
In this paper, we denote by \emph{analyzed WCET} the bound on the WCET computed by a WCET analysis tool, based on assumptions about the underlying hardware. The tool and considered assumptions are detailed in the following. We further denote by \emph{estimated WCET} the WCET predicted by a surrogate model. Neither the analyzed nor the estimated WCET perfectly corresponds to the actual WCET of a program. However, the analyzed WCET is the state-of-the-art approach to reason about the WCET of a program.

We
determine the WCET by static analysis using LLVMTA~\cite{llvmta},
which is a WCET analyzer that
utilizes information obtained from LLVM during the analysis
process. Specifically, for maximum timing predictability and efficiency, our instantiation of LLVMTA assumes a single-core strictly-in-order
(SIC) processor~\cite{Hahn2015,SIC}, which enforces that all hardware
resources are used by instructions in the program order.
This allows a
significantly more efficient pipeline analysis, as timing anomalies
are provably avoided. The SIC processor has the following five-stage pipeline:
instruction fetch, instruction
decode, execute, memory access, and write back.

 The memory hierarchy
consists of separate single-level instruction and data caches, and a
unified background memory.
Both caches share a bus to access the background memory.  Memory
accesses are the result of cache misses
in the instruction fetch
and in the memory access stage.  To enforce strictly in-order execution,
data-cache misses are prioritized over instruction-cache misses.
The caches employ Least-Recently-Used (LRU) replacement, which ensures predictable
behavior~\cite{Reineke07}, even if the initial cache states are
unknown. All
memory levels have a known constant latency.

Our method
focuses on the behavior of conditional
branching and the corresponding effect on the analyzed WCET.
As the processor continues fetching and decoding instructions following branch instructions, conditional branch instructions have a single-cycle latency if the branch is not taken.
On the other hand, if a branch is taken, the pipeline is flushed and the
branch target is fetched, resulting in a longer delay.  In addition to
the immediate cost of flushing the pipeline upon a taken branch, the
branch target is often not in the same cache line as the branch
instruction, which may cause an additional cache miss compared to the
untaken branch.
Thus optimizing conditional branches also implicitly optimizes the implementation's cache performance.


%% file: sections/algo.tex
\rev{1}{Sections 5.1 and 5.2 look contradictory to me regarding the impact of tree depth. It is indicated in section 5.2 that the best results are obtained for deep tree. This does not seem to be true when using synthetic trees in Section 5.1. More explanation is needed here.}

\rev{2}{Equation (3) presents the cost function that needs to be optimized, consisting of a generic overhead cost, a cost for executing the test at each level, and the number of taken branches. It is clear that the overhead and computation per level does not change with the model. So, the only metric that is optimized is the number of taken branches. This does not seem realistic. What about caches? Cache misses are usually orders of magnitude more expensive than conditional branches -- this is also the case for the considered architecture model (100 cycles vs. ~5 cycles?). There seem to be two possible explanations: 1.) the code size at each level is roughly a multiple of the size of a cache line or 2.) the code placement is equally bad across all measurements. In the former case all executions always experience cache misses all the time and there is no possibility for reuse. In the latter case the cache misses become a form a background noise. In both cases a constant per depth cost seems plausible. The authors need to explain why caches are not relevant for the metric ... The fact that the high-level optimization is not able to control the layout is then another issue.}

We provide a pragmatic approach for WCET optimization of decision trees.
As discussed earlier, paths following the untaken slot of a conditional branch are expected to have a lower WCET as they incur fewer instruction cache misses and do not suffer from a pipeline flush.
As a basis for optimization, we thus introduce a surrogate model that estimates a path's execution time based on its length and its number of taken branches.
Although highly complex surrogate models could deliver more accurate predictions, we utilize a linear surrogate model as
the simplicity of this surrogate model allows for a greedy algorithm that determines the optimal realization of a decision tree with respect to the surrogate model, which we describe in Section~\ref{sec:optimization}.
We discuss the surrogate model in more detail in Section~\ref{sec:surrogatemodel} and analyze its accuracy in Section~\ref{sec:surrogatecompare}.

\subsection{A Linear Surrogate Model}
\label{sec:surrogatemodel}

The surrogate model includes two path properties 
that strongly influence a path's WCET: (i)~the total length of the path, since the execution of each node on a path 
contributes to the WCET, and (ii)~the number of taken slots on a path, since a taken branch slot potentially increases the WCET by an instruction cache miss. This leads to the following
WCET estimation for the
parametrized surrogate model
\begin{equation}
	WCET_{surrogate}(d,t)=\sigma+\delta\cdot d + \gamma\cdot t
	\label{surrogatemodel}
\end{equation}
where $d$ is the length or depth of the path and $t$ is the number of taken branches.
Given the parameters $\sigma,~\delta$, and $\gamma$, the model estimates the WCET for a path with the given two properties.

To derive good parameter values for $\sigma,~\delta$, and $\gamma$, we conduct a systematic evaluation of the analyzed WCETs of if-else trees.
Since the determination of the analyzed WCET is not directly dependent on the values encoded in the tree nodes, but rather on the shape of a tree, we consider synthetic trees, which
are not the result of training on any specific data set.

We created a tool to synthesize decision trees up to a given maximal depth and provide statistical information for all paths through the tree.
The tool recursively generates random trees, starting from the tree’s root node, using the following strategy: For each tree generation, as long as the depth of the current node is less than or equal to half the maximal depth of the tree, we recursively generate left and right child nodes. If the depth of the current node is greater than half the maximal depth of the tree, but less than the maximal depth, the tool generates child nodes with a probability of 50\%; otherwise, also with probability 50\%, the node becomes a leaf node.
This procedure intends to mimic the behavior of the decision tree training according to the CART algorithm \cite{breiman1984classification}. This algorithm recursively splits data sets until a stop criterion is reached. One stopping criterion is achieving the maximal configured depth, while high purity in the remaining data can be another criterion. It has to be noted, that these additional stopping criteria only apply when an adequate tree configuration is chosen for the given data set. For instance, if a large data set is used to train a very small decision tree, it likely happens that all path have the maximal depth. Consequently, we generate only trees, which have a shape as trees resulting from an adequate training procedure.
Further, the tool generates an implementation as an if-else-tree in C++ and provides the following  information for every path:
\begin{itemize}[leftmargin=*]
	\item the length of the path, and
	\item the number of taken/untaken branches on the path.
\end{itemize}

For every leaf and the corresponding path, we synthesize input data that forces the inference exactly through this specific path.
Then, we pass the generated tree implementation to LLVMTA and derive the analyzed WCET for each specific path.
The WCET analysis for each single path then forms a data set which is the input for the derivation of the
surrogate model.

\rev{1}{Regarding the number of datasets used train the linear model, I was surprized by the small number of datasets (9) which seems to me very low to train models (even simple models such as linear regression). Perhaps the number of paths is large enough? This point should be commented and the number total number of paths given.}

To make the surrogate model applicable to more scenarios, we generate a set of surrogate models for trees with different maximal depths. Trees with a lower maximal depth likely have paths with similar lengths, compared to trees with higher maximal depths. To accommodate for this properly, we choose a surrogate model for the specific maximal depth of the considered decision tree, which can be directly calculated after the training of the tree.
We generate surrogate models for a maximal depth from $2$ to $18$ in steps of $2$.
When applying the surrogate model to a specific tree, we take the actual maximal depth of this tree and select the surrogate model for the closest rounded down depth.

The maximum depth is limited to $18$ in our surrogate model due to the high computational cost for analyzing the WCET of each path for larger trees.
Specifically, the analysis of the tree with maximal depth $18$ takes up to $35$ minutes for each path on a server with two AMD EPYC 7742 64-core processors with $8$ memory channels each.
The reason for this massive time requirement is that the WCET analyzer is not designed to focus on an individual path through the program.
In particular, the microarchitectural analysis, which explores all reachable parts of the program, is repeated for the analysis of each path.
Path analysis is then focused on the path under the analysis, but the analysis execution time is dominated by the preceding microarchitectural analysis step.
This tree has $4880$ paths, resulting in a total execution time of nearly $8$ days.

After computing the analyzed WCET for the generated synthetic trees and all of their paths,
we perform linear least square fitting of the surrogate model (\Cref{surrogatemodel}) on the synthetic data sets. The resulting parameters are listed in \Cref{tab:surrogatemodels}.
For the case of depth 2, no value for $\delta$ can be found with least square fitting, because there is no variation in the path length. This means, the surrogate model for depth 2 does not include the path length as an argument, which can be represented by setting $\delta$ to $0$.
It should be noted that $\sigma$ is a constant offset in the surrogate model. Although it is important to determine a suitable surrogate model,
it is not specific to properties of a path and therefore does not play a role for optimization.

To investigate the quality of the derived surrogate model, we compute the correlation between the analyzed per path WCET and the estimated WCET by the surrogate model. While an absolute correlation coefficient allows reasoning about absolute surrogate model values, a ranking correlation coefficient allows reasoning about the relative order of the surrogate model outputs.  Hence, we report both an absolute correlation coefficient ($R^2$) and a ranking based correlation coefficient (Kendall's Rank Correlation Coefficient~\cite{665905b2-6123-3642-832e-05dbc1f48979}, often referred to as Kendall's tau) in \Cref{tab:surrogatemodels}.
By construction, Kendall's tau is always between \(-1\) and \(1\), where \(1\) indicates a perfect correlation (i.e., the two rankings are identical), \(-1\) a fully negative correlation (i.e., the two rankings are opposites), and \(0\) no correlation at all.

The surrogate model shows positive correlations between the estimation using the surrogate model and the analyzed WCET derived from LLVMTA when the depth is greater than $2$.
The $R^2$ values are high for all generated target depths.
The lowest value is derived for the depth of $2$, where
the synthetic trees almost all have the same shape. Hence, the dataset
for surrogate-model creation
has low variation in the data points, which reduces the quality of the fitted model.
This is further reflected by Kendall's tau being $0$ for depth of $2$.
Generally, the considered surrogate model can explain the analyzed WCET of the analyzed paths well for all depth larger than $2$. Since decision trees of a depth of $2$ are potentially not interesting and challenging in terms of WCET optimization, this case was not further investigated.

\begin{table}[t]
	\centering
	\caption{Surrogate models derived from the decision trees
          trained on synthesized data sets.
			}
	\begin{tabular}{r r r r r r}
		\toprule
		Depth & $\sigma$ & $\delta$ & $\gamma$  & \( R^2 \) & Kendall's tau \\
		\midrule
		$2$ &   $269.75$    & N/A       & $5.00$    & 0.80 &$0.00$\\
		$4$ &  $226.06$  & $28.84$  & $3.54$        &  $0.98$&$0.86$\\
		$6$ &  $239.40$ & $25.17$  & $5.81$        &  $0.92$&$0.82$\\
		$8$ & $251.84$  & $25.62$  & $8.78$        &  $0.90$&$0.82$\\
		$10$ & $235.53$ & $27.38$  & $8.76$        &  $0.94$&$0.85$\\
		$12$ & $245.21$ & $26.45$  & $11.06$        &  $0.94$&$0.85$\\
		$14$ & $240.58$ & $26.19$  & $11.04$        &  $0.94$&$0.85$\\
		$16$ & $241.08$ & $27.60$  & $9.56$         &  $0.94$&$0.86$\\
		$18$ & $232.68$ & $27.04$  & $10.99$        &  $0.95$&$0.86$\\
		\bottomrule
	\end{tabular}
	\label{tab:surrogatemodels}
	\vspace*{-.5cm}
      \end{table}

\subsection{Accuracy of the Surrogate Model}
\label{sec:surrogatecompare}

As we highlight previously that the chosen surrogate model can serve as a suitable explanation for the analyzed WCET results on the synthetic trees, we investigate the accuracy of the derived surrogate model on real-world datasets.

In our experiment,
we consider decision trees trained on real-world data sets from the
UCI machine learning repository~\cite{asuncion2007uci}, namely adult, bank, covertype, letter, magic, mnist, satlog, sensorless-drive, spambase, and wine-quality.
These data sets are all reasonably large, therefore an adequate training process cannot target very small trees. Consequently, we study a maximal depth of $10$ and $20$ in the following.
We implement these decision trees as they are generated by the training process in a straightforward way without any optimization. In greater detail, we always place the left subtree to the untaken slot and the right subtree to the taken slot.
The decision trees are generated with the arch-forest  framework~\cite{8594826}.
We analyze the WCET of all paths from the root to the leaf through these trees using LLVMTA in the same way that we analyze the paths of the synthetic trees in Section~\ref{sec:surrogatemodel}.

\begin{figure*}[t]
	\centering
	\tikzexternaldisable
	\begin{tikzpicture}
		\begin{axis}[
			xbar,
			symbolic y coords={wine-quality,spambase,sensorless-drive,satlog,mnist,magic,letter,covertype,bank,adult},
			ytick=data,
			yticklabel style={rotate=45,anchor=east,yshift=0.2cm,xshift=-0cm},
			nodes near coords,
			every node near coord/.append style={font=\small},
			width=.51\textwidth,
			height=8cm,
			bar width=0.2cm,
			xmax=1.1,
			legend style={at={(0.9,-0.1)}},
			legend columns=3,
			title=Coefficient of Determination $R^2$
		]
		\addplot[x filter/.code={
				\ifnumequal{\thisrow{depth}}{10}{}{\def\pgfmathresult{}}
			}, fill=green!40!white] table [y=dataset, x=r2, col sep=comma] {plots/surrogatemodel/accuracy.csv};
		\addplot[x filter/.code={
				\ifnumequal{\thisrow{depth}}{20}{}{\def\pgfmathresult{}}
			}, fill=red!40!white] table [y=dataset, x=r2, col sep=comma] {plots/surrogatemodel/accuracy.csv};
			\legend{Depth 10\hspace*{.2cm}, Depth 20}
		\end{axis}
	\end{tikzpicture}
	\hspace*{-0.75cm}
	\begin{tikzpicture}
		\begin{axis}[
			xbar,
			symbolic y coords={wine-quality,spambase,sensorless-drive,satlog,mnist,magic,letter,covertype,bank,adult},
			yticklabel style={rotate=45,anchor=east,yshift=0.2cm,xshift=-0cm},
			yticklabels=\empty,
			ytick=data,
			nodes near coords,
			every node near coord/.append style={font=\small},
			width=.5\textwidth,
			height=8cm,
			bar width=0.2cm,
			xmax=1.1,
			legend style={at={(0.9,-0.1)}},
			legend columns=3,
			title=Kendall's Rank Correlation Coefficient (Kendall's tau)
		]
		\addplot[x filter/.code={
				\ifnumequal{\thisrow{depth}}{10}{}{\def\pgfmathresult{}}
			}, fill=green!40!white] table [y=dataset, x=kt, col sep=comma] {plots/surrogatemodel/accuracy.csv};
		\addplot[x filter/.code={
				\ifnumequal{\thisrow{depth}}{20}{}{\def\pgfmathresult{}}
			}, fill=red!40!white] table [y=dataset, x=kt, col sep=comma] {plots/surrogatemodel/accuracy.csv};
		\legend{Depth 10\hspace*{.2cm}, Depth 20}
		\end{axis}
	\end{tikzpicture}
	\caption{Determination of the Surrogate Model for Real Datasets}
	\vspace*{-0.5cm}
	\label{r2realdatasets}
      \end{figure*}
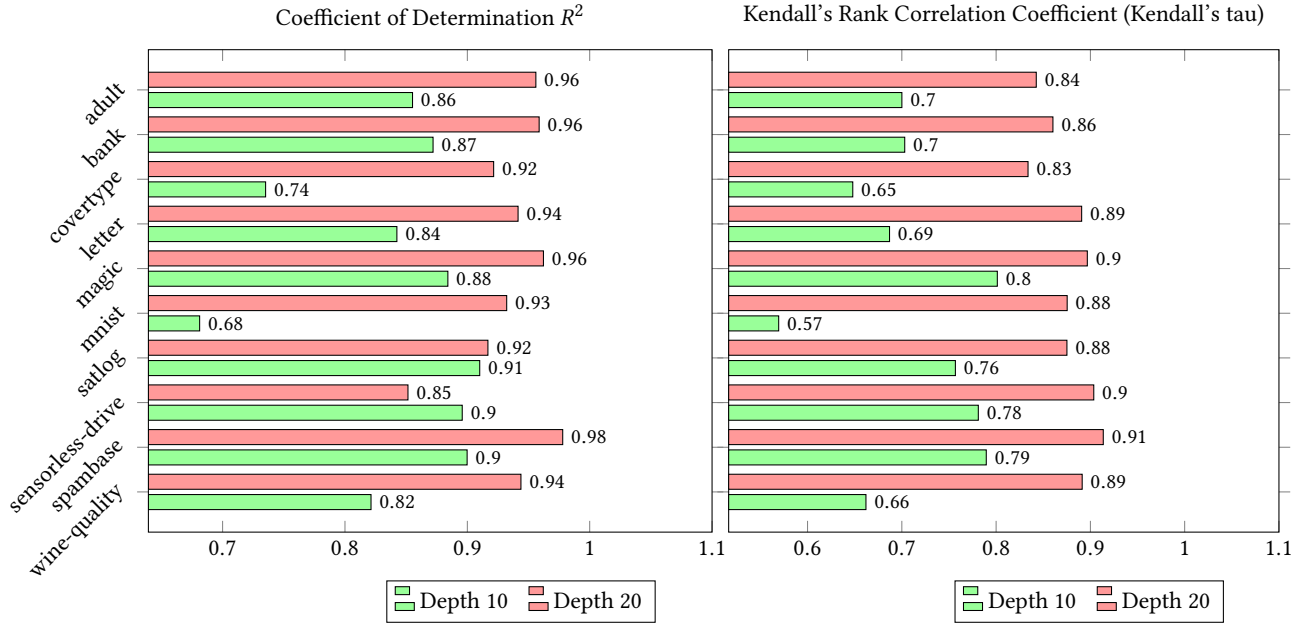

Since the surrogate model aims to serve as a compositional model for recursive  construction of an optimal solution, we investigate the accuracy of the absolute predictions of the surrogate model compared to WCET analysis.
Since the conducted experiment cannot cover all possible path combinations due to the massive computational demands, investigating the correlation in terms of absolute values helps to assess the extrapolation of the surrogate model beyond this experiment. Therefore,
we condense the difference by evaluating the coefficient of determination ($R^2$) between the surrogate model prediction and the analyzed WCET of each path. Furthermore, since the target of WCET optimization is to derive a tree with a minimized WCEP, the ranking of predictions by the surrogate model plays an important role, since only the path with the highest rank forms the WCEP.
Hence, we additionally investigate the ranking correlation between predictions from the surrogate model and the analyzed WCET in the form of Kendall's rank correlation coefficient (Kendall's tau).
The resulting values are illustrated in \Cref{r2realdatasets}.

\Cref{r2realdatasets}, on the x axis, depicts the resulting $R^2$ values
in the left plot and the Kendall's tau values
in the right plot, respectively. For each dataset, one bar for a real tree with a maximal depth of 10 and 20 is shown.
The studied trees show
 high correlation values.
Naturally, the prediction quality of the surrogate model depends on the considered dataset. For instance, the mnist dataset
has a significantly lower correlation in terms of absolute values for trees of a maximal depth of 10 than the satlog dataset.
The reason is that according to the distribution of observations in the dataset, training results in differently shaped trees. The surrogate model is created on a generalized variation of tree shapes, which naturally cannot apply equally well to all real scenarios.
Nevertheless, the linear compositional surrogate model provides
a good estimation of the analyzed WCET in this experiment.
The strong correlation in terms of ranking highlights the surrogate model's
robustness
when identifying the WCEP.

\section{Algorithmic Optimization}
\newcommand{\taken}{\textit{taken}}
\newcommand{\untaken}{\textit{untaken}}
\newcommand{\depth}{\textit{depth}}
\label{sec:optimization}

\rev{1}{at the very start of section 6, it would be nice to relate the introduced variables with the ones used in equation (3) (simply say that delta and tau represent the same variables as in equation 3, it takes some time figuring out they are the same variables, because you used different terms to define them in sections 5 and 6).}

Since the surrogate model developed in Section~\ref{sec:surrogate} is linear and only dependent on local path parameters, we can utilize
it
in the WCET optimization of decision trees.
Consider a hypothetic example, where the cost factor for the depth is $\delta=2$, the cost for the taken slots is $\gamma=1$, and the constant $\sigma$ is $0$. \Cref{fig:badexp} depicts a simple decision tree, which is implemented na\"ively, without any optimization of branch conditions. The unoptimized implementation results in an estimated WCET of $6$. \Cref{fig:goodexp} shows an example of an optimized implementation of the same tree, where the root node has a flipped branch condition. The surrogate model output for the left subtree under the root node is $2$, for the right subtree $3$. Thus, when flipping the branch condition, the overall estimated WCET is reduced to $5$.

\begin{figure}[t]

  \subfloat[Unoptimized taken/untaken blocks]{
    \label{fig:badexp}
    \centering
    \scalebox{0.7}{
      \begin{tikzpicture}
        \node[draw, circle, minimum width=0.5cm] (r) at (0,0) {};

        \node[draw, circle, minimum width=0.5cm] (rl) at (-2,-2) {};
        \node[draw, circle, minimum width=0.5cm] (rr) at (2,-2) {};

        \node[draw, circle, minimum width=0.5cm] (rrl) at (1,-4) {};
        \node[draw, circle, minimum width=0.5cm] (rrr) at (3,-4) {};

        \draw[->] (r) --node[above, yshift=.3cm, xshift=-.1cm] {Untaken}(rl);
        \draw[->] (r) --node[above, yshift=.3cm] {Taken}(rr);

        \draw[->] (rr) --node[above, yshift=.1cm, xshift=-.4cm] {Untaken}(rrl);
        \draw[->] (rr) --node[above, yshift=.1cm, xshift=.4cm] {Taken}(rrr);
      \end{tikzpicture}
    }
  }
  \subfloat[Optimized taken/untaken blocks]{
    \label{fig:goodexp}
    \centering
    \scalebox{0.7}{
      \begin{tikzpicture}
        \node[draw, circle, minimum width=0.5cm] (r) at (0,0) {};

        \node[draw, circle, minimum width=0.5cm] (rl) at (-2,-2) {};
        \node[draw, circle, minimum width=0.5cm] (rr) at (2,-2) {};

        \node[draw, circle, minimum width=0.5cm] (rrl) at (1,-4) {};
        \node[draw, circle, minimum width=0.5cm] (rrr) at (3,-4) {};

        \draw[->] (r) --node[above, yshift=.3cm] {Taken}(rl);
        \draw[->] (r) --node[above, yshift=.3cm, xshift=.1cm] {Untaken}(rr);

        \draw[->] (rr) --node[above, yshift=.1cm, xshift=-.4cm] {Untaken}(rrl);
        \draw[->] (rr) --node[above, yshift=.1cm, xshift=.4cm] {Taken}(rrr);
      \end{tikzpicture}
    }
  }
  \caption{Optimization of Decision Trees}
  \label{fig:bad+goodexp}
\end{figure}
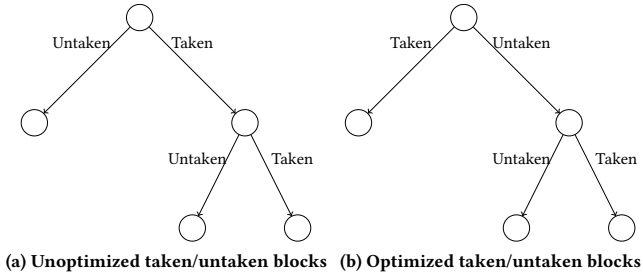

Let $G=(V,E)$ be a directed graph representing the binary decision tree,
where $V$ is the set of the nodes and $E$ is the set of the directed
edges. We call the node
$v \in V$ with no incoming edges
the root,  and a node $v \in V$ with no outgoing edges a leaf.
Every
binary decision tree has exactly one root.
For brevity, let $ST(v)$ be the subtree rooted at $v$.
That is, if $v$ is a leaf, then $ST(v)$ only contains $v$
without any outgoing edge, and if $v$ is the root, then $ST(v)$ is
the graph~$G$.

Let $\lambda: E \rightarrow \{T, U\}$ be a labeling function which
labels a directed edge $e \in E$ as a taken slot $T$ or an untaken $U$
slot. By the definition of a binary decision tree, for a node
$v \in V$ with two children nodes $u, u' \in V$, either
\begin{enumerate}
\item $\lambda((v, u))$ is $T$ and $\lambda((v, u'))$ is $U$, or
\item $\lambda((v, u))$ is $U$ and $\lambda((v, u'))$ is $T$.
\end{enumerate}

Let $L(ST(v))$ be the set of leaves of a rooted subtree $ST(v)$, i.e.,
$L(ST(v)) = \{\ell \mid \ell \in ST(v) \mbox{ and } \ell \mbox{ is a leaf}\}$.  Let $r$ be
the root of $G$. For a rooted tree, there is one unique path from $r$
to a leaf $\ell \in L$. Let $\depth(p, \ell)$ be the depth (i.e., number of
edges) of the path from $p \in V$ to $\ell \in L$ and let
$taken(\lambda, p, \ell)$ be the number of edges in the path from $p$
to $\ell$ which are labeled as taken in the labeling function~$\lambda$.

The WCET optimization problem based on our linear surrogate model is
to find a labeling function $\lambda$ such that the estimated WCET $C(ST(r), \lambda)$ is
minimized, where
\begin{equation}\small
  C(ST(r), \lambda) = \sigma+\max_{\ell \in L(ST(r))} \left\{
    \begin{array}{ll}
        & \delta \cdot \depth(r, \ell)\\
      + &\gamma \cdot \taken(\lambda, r, \ell)
    \end{array}  \right\}.\label{eq:wcet-surrogate}
\end{equation}
By definition, the depth of a path is equal to the number of taken edges plus the number of untaken edges.
Let $untaken(\lambda, v, \ell)$ be $\depth(v, \ell) - taken(\lambda, v, \ell)$.
With $\pi=\delta+\gamma$, we can reformulate Eq.~(\ref{eq:wcet-surrogate}) into
\begin{equation}\small
  C(ST(r), \lambda) = \sigma+\max_{\ell \in L(ST(r))} \left\{
    \begin{array}{ll}
        & \delta \cdot \untaken(\lambda, v, \ell)\\
      + &\pi \cdot \taken(\lambda, v, \ell)
    \end{array}
  \right\}.
  \label{eq:wcet-surrogate-reduced-pi}
\end{equation}
As in our surrogate model, $\delta \geq 0$ and $\gamma \geq 0$,
thus, $\pi \geq \delta \geq 0$.

Let the root $r$ have two children $(c_1, c_2)$ in $G$. For a
given~$\lambda$, based on
(\ref{eq:wcet-surrogate-reduced-pi}),
the WCET of the
labeling function $\lambda$ is
\begin{eqnarray}\small
  \label{eq:wcet-surrogate-expand-case1}
  C(ST(r), \lambda) = \sigma + \max\left\{
	\begin{array}{ll}
		& \delta + C(ST(c_1), \lambda)\\
  		+ &\pi+
		  C(ST(c_2), \lambda)
	\end{array}\right\}\\
   \mbox{if $\lambda((r,c_1))$ is untaken and
$\lambda((r,c_2))$ is taken} \nonumber
\end{eqnarray}
or
\begin{eqnarray}
  \label{eq:wcet-surrogate-expand-case2}
  C(ST(r), \lambda) = \sigma + \max\left\{
	\begin{array}{ll}
		& \pi + C(ST(c_1), \lambda)\\
  		+ &\delta+
		  C(ST(c_2), \lambda)
	\end{array}\right\}\\
   \mbox{if $\lambda((r,c_1))$ is taken and
$\lambda((r,c_2))$ is untaken}  \nonumber
\end{eqnarray}

As there are only two options to label the two edges $(r, c_1)$ and
$(r, c_2)$, with given $C(ST(c_1), \lambda)$ and $C(ST(c_2), \lambda)$, the optimal
decision to label $(r, c_1)$ and $(r, c_2)$ is to find the minimum
solution between
Eq.~(\ref{eq:wcet-surrogate-expand-case1})~and~Eq.~(\ref{eq:wcet-surrogate-expand-case2}). That
is, if Eq.~(\ref{eq:wcet-surrogate-expand-case1}) is smaller than
Eq.~(\ref{eq:wcet-surrogate-expand-case2}), then $\lambda((r,c_1))$
should be untaken and $\lambda((r,c_2))$ should be taken; otherwise,
$\lambda((r,c_1))$ should be taken and $\lambda((r,c_2))$ should be
untaken.

Suppose that we have two labeling functions  already: (i)~$\lambda_1$ for the
subtree rooted at $c_1$ and (ii)~$\lambda_2$ for the subtree rooted at~$c_2$.  The following lemma shows that the labeling decisions
for the edges $(r, c_1)$ and $(r, c_2)$ can be greedily done by
comparing $C(ST(c_1), \lambda_1)$ and $C(ST(c_2), \lambda_2)$.
\begin{lemma}
  \label{lemma:subopt}
  For a given surrogate model of the WCET defined in
  Eq.~(\ref{eq:wcet-surrogate-reduced-pi}) with
  $\pi \geq \delta \geq 0$ and two labeling functions $\lambda_1$ for the
  subtree rooted at $c_1$ and $\lambda_2$ for the subtree rooted at
  $c_2$, the condition
  \[
    C(ST(c_1), \lambda_1) \leq C(ST(c_2), \lambda_2)
  \]
  implies that
  \begin{eqnarray}
    \label{eq:binary-relation}
     & \max\left\{\pi + C(ST(c_1), \lambda_1), \delta+ C(ST(c_2), \lambda_2)\right\} \nonumber\\
    \leq & \max\left\{\delta + C(ST(c_1), \lambda_1), \pi+ C(ST(c_2), \lambda_2)\right\}.
  \end{eqnarray}
  Therefore, if $C(ST(c_1), \lambda_1) \leq C(ST(c_2), \lambda_2)$,
  then $\lambda((r,c_1))$ should be taken and $\lambda((r,c_2))$
  should be untaken; otherwise $\lambda((r,c_1))$ should be untaken
  and $\lambda((r,c_2))$ should be taken. Ties can be
  broken arbitrarily. The subtree rooted at $c_1$ ($c_2$, respectively) is labeled exactly as
  $\lambda_1$ ($\lambda_2$, respectively).
\end{lemma}
\begin{proof}
  This comes from the definition and simple arithmetic of inequalities
  by the condition $C(ST(c_1), \lambda_1) \leq C(ST(c_2), \lambda_2)$ and the assumption
  $\pi \geq \delta \geq 0$.
\end{proof}

\newcommand{\algname}{\textsc{SurrogateOpt}}

\begin{algorithm}[t]
	\begin{algorithmic}[1]
		\STATE \textbf{Given:} A binary tree $G = (V, E)$ rooted at $r \in V$.
		\STATE \textbf{Output:} A labeling $\lambda$ that minimizes the tree's WCET based on the surrogate model.
		\IF{$r$ is leaf node}
		\RETURN empty labeling
		\ELSE
		\STATE let the two children of $r$ be $c_1$ and $c_2$ in $V$
		\STATE let $\lambda_1 = \algname(ST(c_1), c_1)$
		\STATE let $\lambda_2 = \algname(ST(c_2), c_2)$
		\STATE let $\lambda = \lambda_1 \cup \lambda_2$
                \IF{$C(ST(c_1), \lambda_1) \leq C(ST(c_2), \lambda_2)$ )} \label{algo-opt-condition}
                \RETURN  $\lambda \cup \{(r, c_1) \mapsto T, (r, c_2) \mapsto U\}$
                \ELSE
                \RETURN $\lambda \cup \{(r, c_2) \mapsto T, (r, c_1) \mapsto U\}$
		\ENDIF
		\ENDIF
	\end{algorithmic}
	\caption{$\algname(G, r)$}
	\label{algo:mapnodeoptimal-rev}
      \end{algorithm}

With the above analysis, we reach an optimal greedy strategy which
constructs the labels from the leaves in $L(ST(r))$ back to the root
$r$. Algorithm~\ref{algo:mapnodeoptimal-rev} illustrates the
pseudocode in a recursive form.

We prove the optimality in the following.

\begin{theorem}[Best-Case Labeling]
  \label{thm:best-case}
  Suppose a given surrogate model of the WCET model defined in
  Eq.~(\ref{eq:wcet-surrogate}). \Cref{algo:mapnodeoptimal-rev}
  derives a labeling function $\lambda$ that minimizes the estimated WCET
  defined in Eq.~(\ref{eq:wcet-surrogate}).
\end{theorem}
\begin{proof}[Proof]
  The optimality
   can be proven by a standard swapping argument. That
  is, for any labeling function of the tree, we can use
  Lemma~\ref{lemma:subopt} to show that swapping from the given labeling
  function towards our derived labeling function does not increase the
  estimated WCET defined in Eq.~(\ref{eq:wcet-surrogate-reduced-pi}).
\end{proof}

Interestingly, we can also show that the structure of
Lemma~\ref{lemma:subopt} leads to the following corollary to
maximize the worst-case execution time, which can be used to
demonstrate the potential of WCET optimization.
\begin{corollary}[Worst-Case Labeling]
  \label{cor:worst-case}
  By inverting the greedy decision of
  Algorithm~\ref{algo:mapnodeoptimal-rev}, i.e., by changing the
  condition in Line~\ref{algo-opt-condition} to
  $C(ST(c_1), \lambda_1) > C(ST(c_2), \lambda_2)$ instead, the revised algorithm derives a
  labeling function $\lambda^{worst}$ that maximizes the estimated WCET defined
  in Eq.~(\ref{eq:wcet-surrogate}).
\end{corollary}


%% file: sections/eval.tex
\rev{1}{In the performance evaluation, if would be interesting to compare (for time considerations on one tree only) the proposed technique with a technique that explores all paths and evaluates their WCETs (to confirm that the linear model is effectively sufficient to make the optimizations).}

After showing the optimality of the proposed decision tree optimization algorithm w.r.t. the surrogate model, we provide experimental evaluation with a focus on the relation to the analyzed WCET of the corresponding decision trees. We further compare the optimized implementation with existing ACET-optimization algorithms from the literature.

We consider a variety of trained decision trees, as in \Cref{sec:surrogatecompare}.
We use the arch-forest  framework~\cite{8594826} to train decision trees on different data sets, namely adult, bank, covertype, letter, magic, mnist, satlog, sensorless-drive, spambase, and wine-quality from the UCI
 ML repository \cite{asuncion2007uci}.

We train the decision trees without specific hyperparameter tuning.
For each data set, we train seven trees, each with a specified maximal
depth of \(1\), \(5\), \(10\), \(20\), \(30\), \(40\), and \(50\).
Although these datasets train adequately on large trees, as stated before, we cover a broader range of maximal depths. In the following results, we consider all trees with at least a maximal depth of 10 as an adequate shape. We do not consider maximal depths of more than 50, since the considered data sets do not result in deeper real depths.
We note that the trees do not necessarily reach their maximum allowed depth and balance due to the randomized procedures in the training process, and consequently their sizes and balance differ.
Since the maximal depth is not always reached during the training process, especially for smaller data sets, we refer to the actual depth in the following discussions.

Such trained trees are available as high-level-specified data structures, which are then implemented in C code by arch-forest. In addition to a set of implementations, provided by arch-forest, we implement the surrogate model optimization algorithm as a new code generator in arch-forest.
Those implementations are analyzed by LLVMTA, which uses static analysis to derive the analyzed WCET of the decision tree inference implementations. Consequently, LLVMTA considers all possible paths through the code of a tree and determines the analyzed WCET by identifying the WCEP.
The inference of specific data points has no impact, since we are only interested in the worst case.
We compare to the following decision tree implementations:

\begin{itemize}[leftmargin=*]
    \item \textbf{Standard} is a straightforward if-else tree implementation of the decision tree, referred to as codegen in the literature. This implementation is provided under the name Standard by arch-forest.
    Nodes are translated to if-else-blocks, the left subtree is always mapped to the if-block and the right subtree is always mapped to the else-block.
    No further
    optimization, e.g.,~flipping of the
    if-block and the else-block, is applied.
    \item \textbf{Chen et~al. Swap}~\cite{chen2022efficient} is an optimization of the aforementioned \textbf{Standard}.
    It considers knowledge about prefetching and preemption in the CPU caches and
    arranges paths of the decision tree in a cache-friendly manner w.r.t. the probabilistic distribution of data points in the training dataset.
    These probabilities are used to group highly probable nodes together in cache lines to improve cache locality.
    This implementation is motivated by ACET, since it optimizes the most probable case for the given data set.
    It builds on the fact that the tree is inferred repetitively, and highly probable nodes are not evicted from the cache and therefore exhibit low access latencies.
    The algorithm only flips branch conditions to achieve the aforementioned target. Hence, it utilizes the same tuning knobs as our proposed optimization algorithm.
    \item \textbf{Chen et~al. Path}~\cite{chen2022efficient} is an extension
    of the {Swap} version,
    where the length of jumps on a taken branch is also considered as an optimization target. The optimization builds an if-else tree as a kernel until a certain capacity is exceeded.
    The remaining subtrees are implemented at a disjoint memory location and  additional branches are introduced to maintain logical correctness.
    This optimization uses an additional knob of implementation tuning, which, however, due to the considered additional cost for taken branches in the surrogate model, is not a
    suitable candidate for the surrogate-model-based optimization algorithm.
    We note that this optimization results in exactly the same implementation as the Swap version, as long as the entire tree fits into the considered kernel, which happens
    for all considered cases in this evaluation, except the trees generated
    for the covertype dataset.
    \item \textbf{SurrogateOpt}
    implements \Cref{algo:mapnodeoptimal-rev}
    using the presented surrogate model. Evaluated trees
    where the actual depth exceeds $18$ are applied to the surrogate model for depth $18$, as described in \Cref{sec:surrogatemodel}.
    \item \textbf{Inverted}, defined in Corollary~\ref{cor:worst-case}, implements the same optimization as \textbf{SurrogateOpt}, but inverts the decision of the surrogate model. Instead of minimizing a tree's estimated WCET,
    it is maximized,
    providing an upper bound w.r.t. the surrogate model.
\end{itemize}

We adopt the surrogate model designed in
Section~\ref{sec:surrogatemodel} for depth
$2, 4, 6, 8, 10, 12, 14, 16, 18$. For a given decision tree with depth
$d$, we pick the corresponding surrogate model by flooring the actual
maximal depth to the next surrogate model.

\subsection{Results}
\begin{figure*}
    \centering
    \begin{minipage}{.48\textwidth}
        \includegraphics[width=\textwidth]{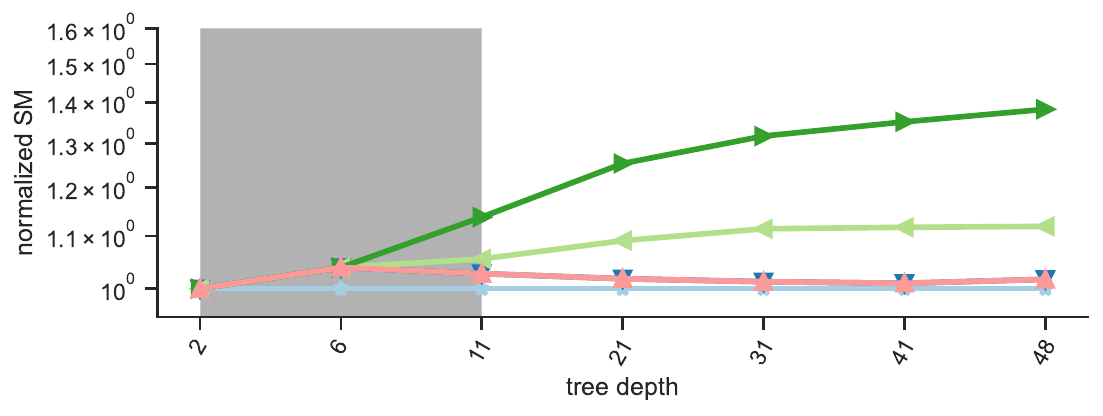}
        \vspace{-.8cm}
        \begin{center}
            adult
        \end{center}
        \vspace{-.1cm}
    \end{minipage}
    \begin{minipage}{.48\textwidth}
        \includegraphics[width=\textwidth]{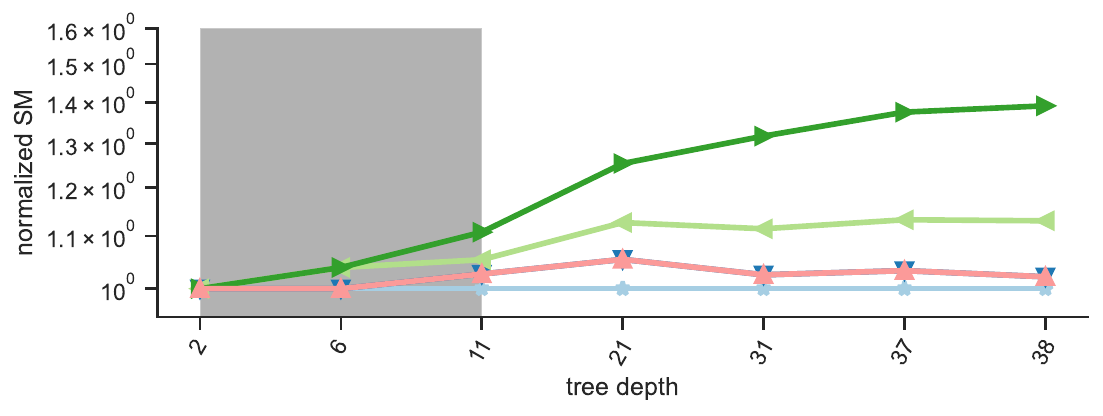}
        \vspace{-.8cm}
        \begin{center}
            bank
        \end{center}
        \vspace{-.1cm}
    \end{minipage}
    \begin{minipage}{.48\textwidth}
        \includegraphics[width=\textwidth]{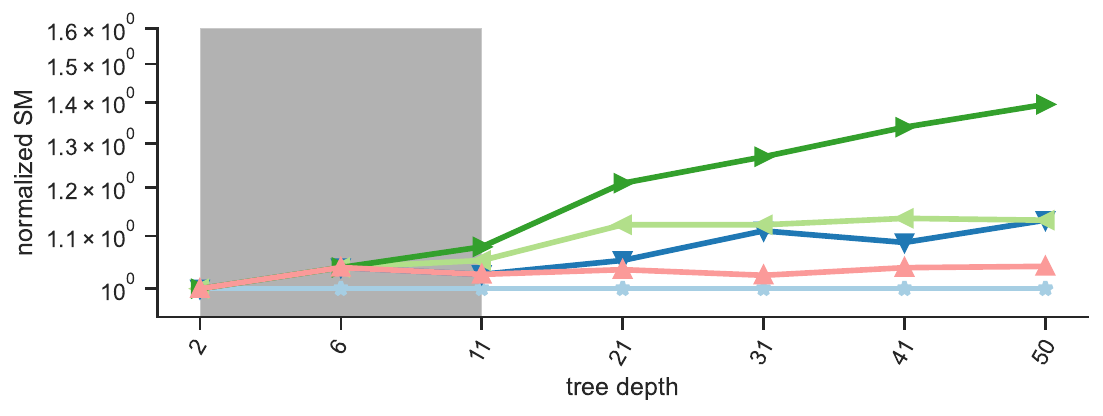}
        \vspace{-.8cm}
        \begin{center}
            covertype
        \end{center}
        \vspace{-.1cm}
    \end{minipage}
    \begin{minipage}{.48\textwidth}
        \includegraphics[width=\textwidth]{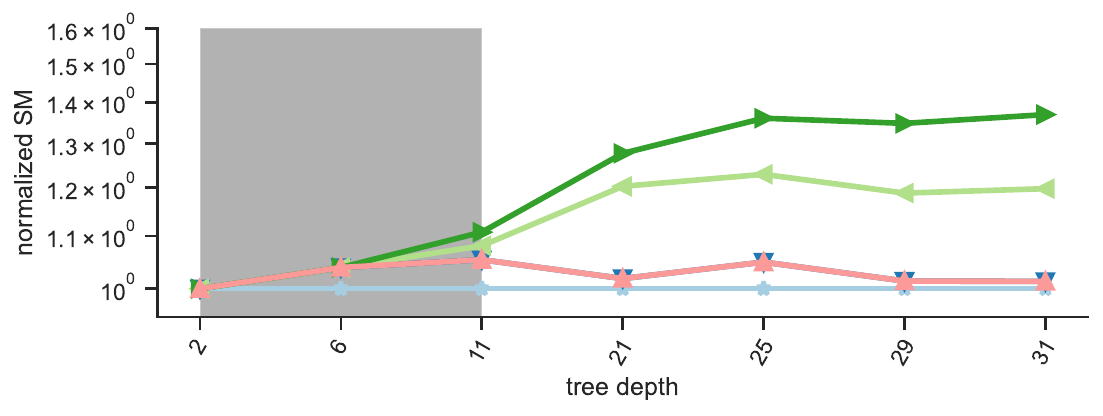}
        \vspace{-.8cm}
        \begin{center}
            letter
        \end{center}
        \vspace{-.1cm}
    \end{minipage}
    \begin{minipage}{.48\textwidth}
        \includegraphics[width=\textwidth]{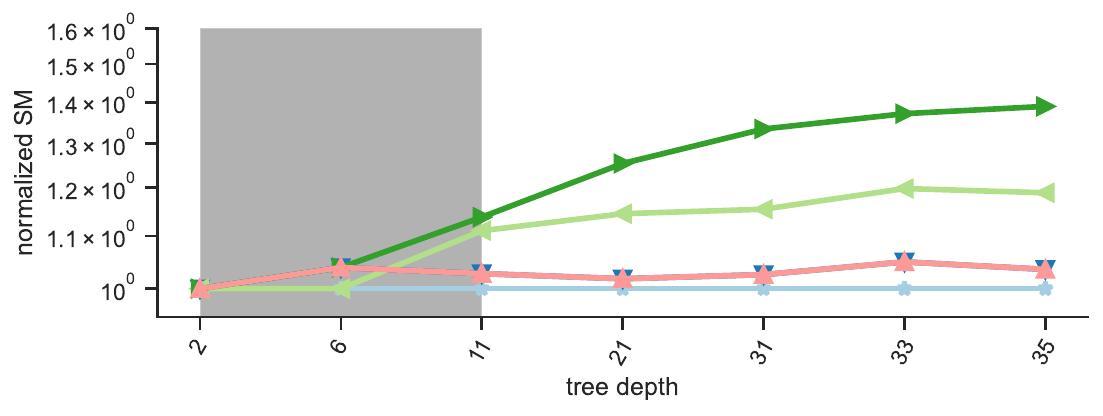}
        \vspace{-.8cm}
        \begin{center}
            magic
        \end{center}
        \vspace{-.1cm}
    \end{minipage}
    \begin{minipage}{.48\textwidth}
        \includegraphics[width=\textwidth]{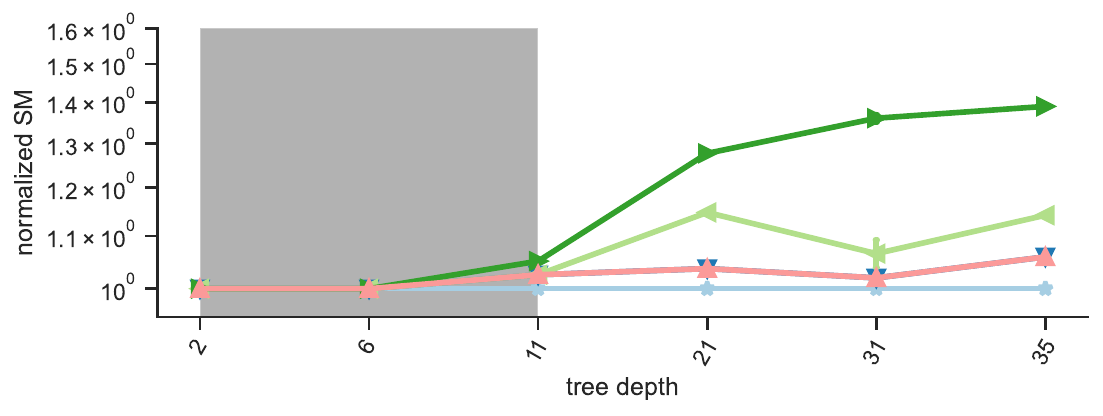}
        \vspace{-.8cm}
        \begin{center}
            mnist
        \end{center}
        \vspace{-.1cm}
    \end{minipage}
    \begin{minipage}{.48\textwidth}
        \includegraphics[width=\textwidth]{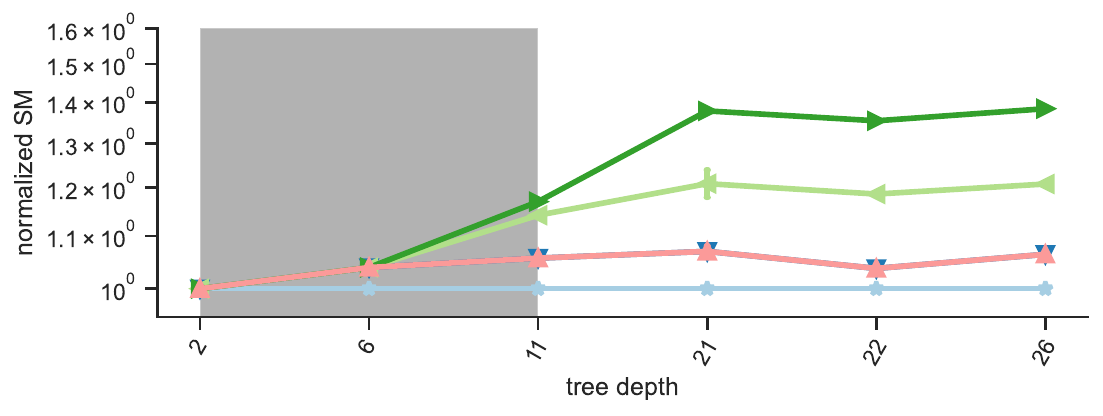}
        \vspace{-.8cm}
        \begin{center}
            satlog
        \end{center}
        \vspace{-.1cm}
    \end{minipage}
    \begin{minipage}{.48\textwidth}
        \includegraphics[width=\textwidth]{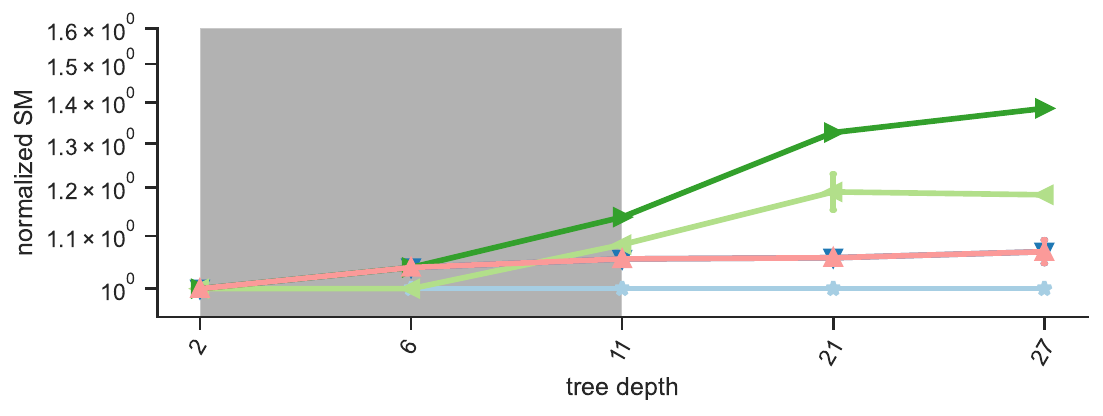}
        \vspace{-.8cm}
        \begin{center}
            sensorless-drive
        \end{center}
        \vspace{-.1cm}
    \end{minipage}
    \begin{minipage}{.48\textwidth}
        \includegraphics[width=\textwidth]{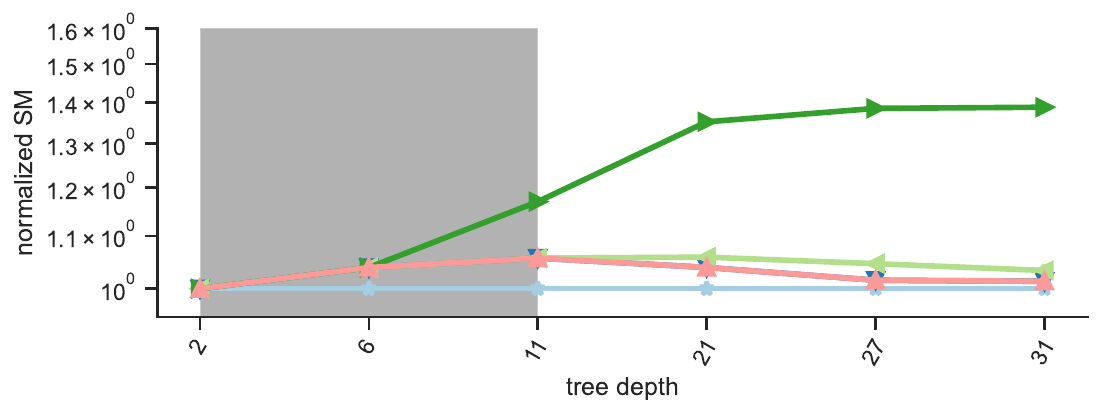}
        \vspace{-.8cm}
        \begin{center}
            spambase
        \end{center}
        \vspace{-.1cm}
    \end{minipage}
    \begin{minipage}{.48\textwidth}
        \includegraphics[width=\textwidth]{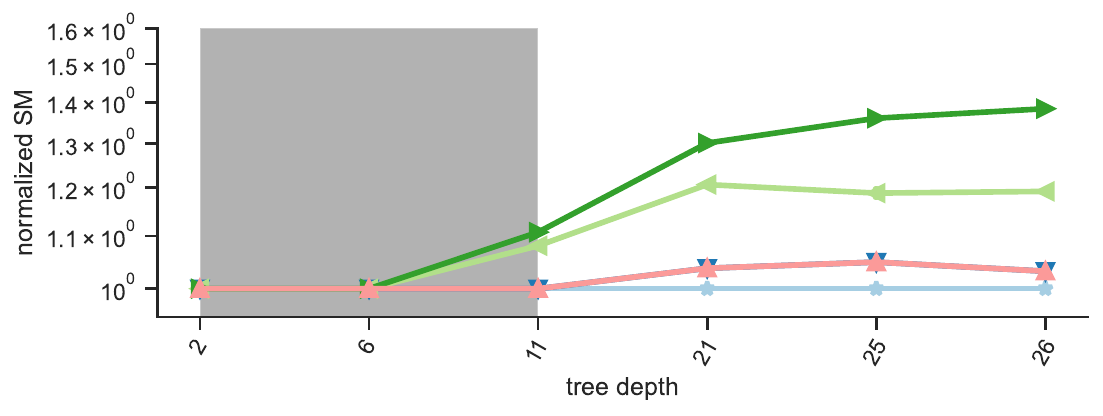}
        \vspace{-.8cm}
        \begin{center}
            wine-quality
        \end{center}
        \vspace{-.1cm}
    \end{minipage}

    \begin{minipage}{.6\textwidth}
        \includegraphics[width=\textwidth]{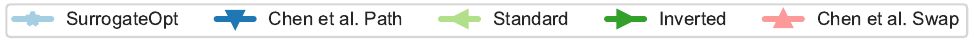}
    \end{minipage}
    \caption{Normalized estimated WCET based on the surrogate model (denoted as Normalized SM) for all datasets.}
    \label{fig:surr}
\end{figure*}

Before studying the analyzed WCET by LLVMTA, we apply the presented surrogate model to the various implementations mentioned before. The results are depicted in \Cref{fig:surr}. Each dataset is illustrated by one plot. The x axis depicts the increasing actual tree depth, while the y axis shows the normalized estimated WCET, w.r.t. the surrogate model, each tree achieves in the corresponding implementation. The gray area indicates the trees, which are considered to not have an adequate shape for the data set.
We denote the normalized results derived by the surrogate model as normalized SM (SurrogateModel) in \Cref{fig:surr}.
For the implementation of Chen et~al. Path,
when a node is not in a kernel, we consider $\delta+\gamma$ to be the estimated cost for the taken and the untaken slot. We normalize all values to the corresponding value of the SurrogateOpt implementation, since its value is proven to be minimal.
Furthermore, the data series for the inverted optimization indicates an upper bound.

The Standard implementation leads to $\approx 20\%$ worse values
for the surrogate model in comparison to the optimized implementation for many cases with deep trees. For smaller trees, all implementations tend to converge in their cost of the surrogate model, which can be caused by the fact that very small trees tend to have only paths of the same length, hence, there is no optimization potential.
The ACET optimization strategy from the literature (Chen et~al.) can improve the estimated WCET w.r.t. the surrogate model in comparison to the Standard implementation in almost all cases, however, still underperforms compared to the proposed optimization algorithm. The Path version, which introduces additional branches out of the kernel, naturally results in a worse estimated WCET for the covertype dataset.

\begin{figure*}
    \centering
    \begin{minipage}{.48\textwidth}
        \includegraphics[width=\textwidth]{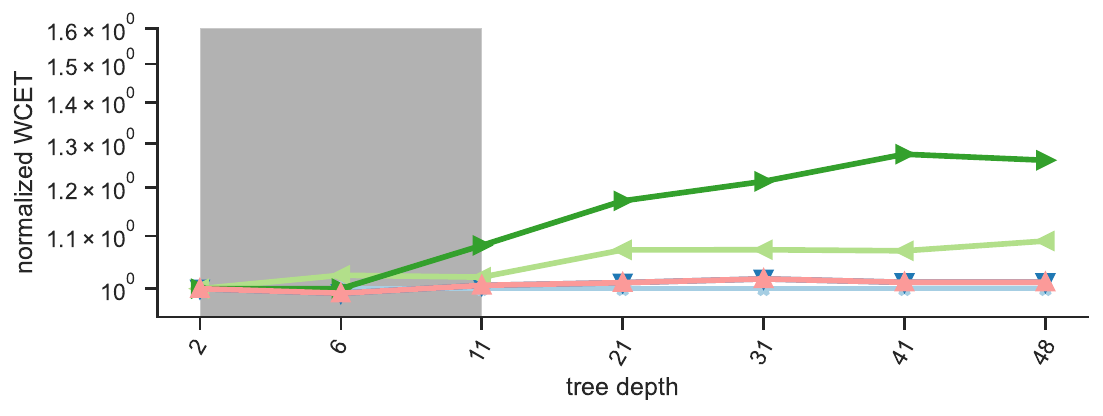}
        \vspace{-.8cm}
        \begin{center}
            adult
        \end{center}
        \vspace{-.1cm}
    \end{minipage}
    \begin{minipage}{.48\textwidth}
        \includegraphics[width=\textwidth]{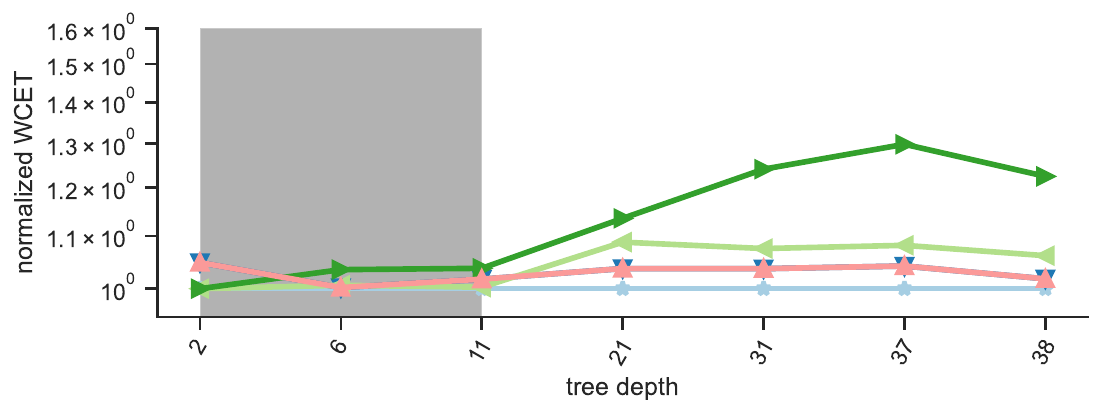}
        \vspace{-.8cm}
        \begin{center}
            bank
        \end{center}
        \vspace{-.1cm}
    \end{minipage}
    \begin{minipage}{.48\textwidth}
        \includegraphics[width=\textwidth]{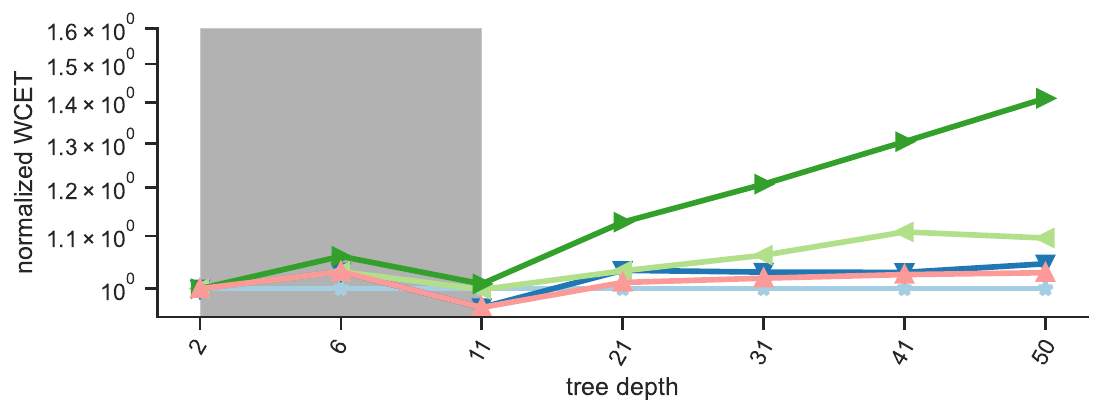}
        \vspace{-.8cm}
        \begin{center}
            covertype
        \end{center}
        \vspace{-.1cm}
    \end{minipage}
    \begin{minipage}{.48\textwidth}
        \includegraphics[width=\textwidth]{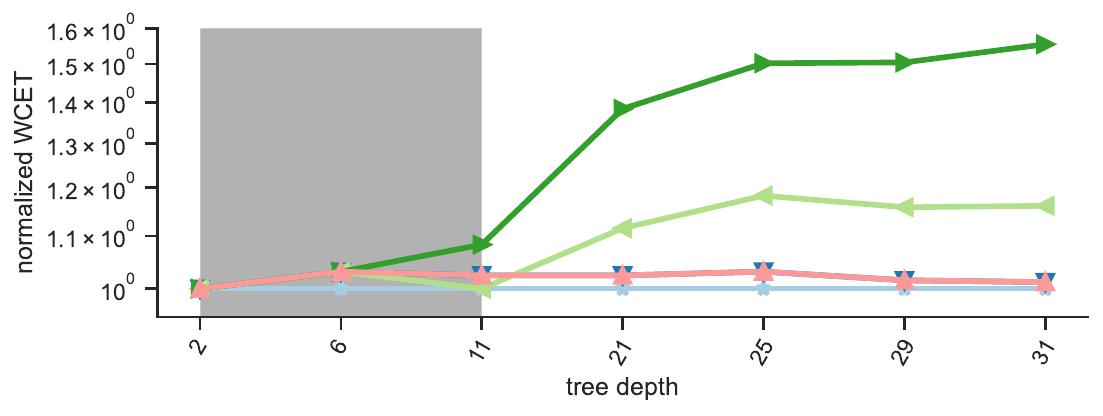}
        \vspace{-.8cm}
        \begin{center}
            letter
        \end{center}
        \vspace{-.1cm}
    \end{minipage}
    \begin{minipage}{.48\textwidth}
        \includegraphics[width=\textwidth]{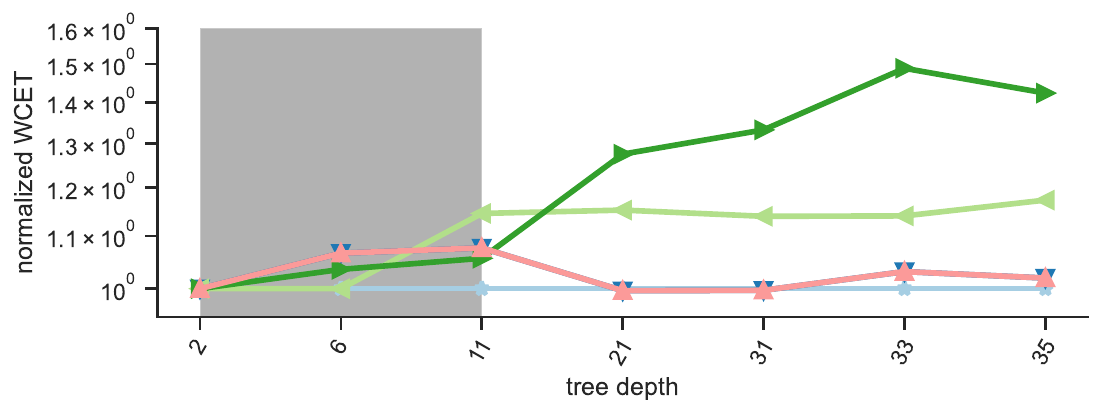}
        \vspace{-.8cm}
        \begin{center}
            magic
        \end{center}
        \vspace{-.1cm}
    \end{minipage}
    \begin{minipage}{.48\textwidth}
        \includegraphics[width=\textwidth]{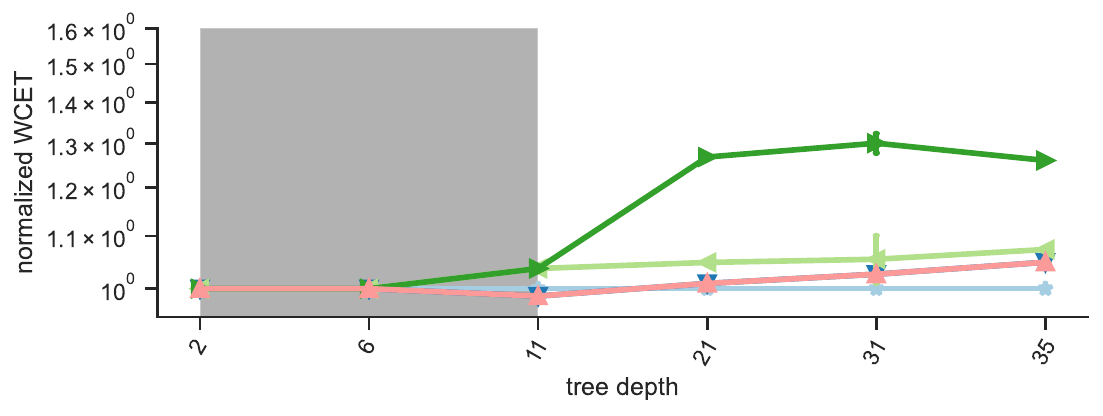}
        \vspace{-.8cm}
        \begin{center}
            mnist
        \end{center}
        \vspace{-.1cm}
    \end{minipage}
    \begin{minipage}{.48\textwidth}
        \includegraphics[width=\textwidth]{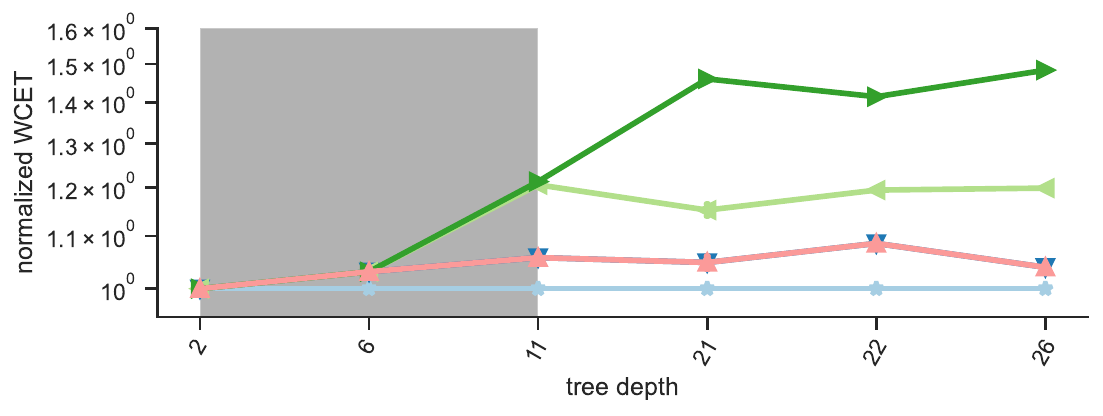}
        \vspace{-.8cm}
        \begin{center}
            satlog
        \end{center}
        \vspace{-.1cm}
    \end{minipage}
    \begin{minipage}{.48\textwidth}
        \includegraphics[width=\textwidth]{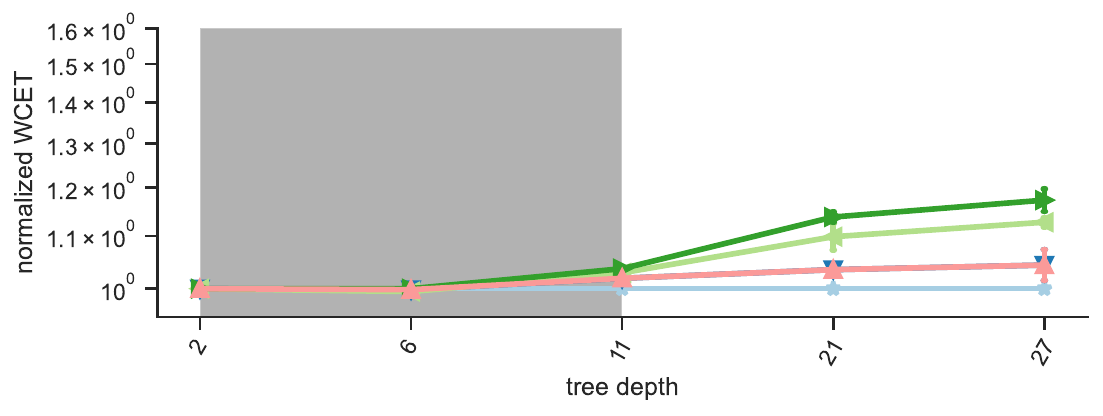}
        \vspace{-.8cm}
        \begin{center}
            sensorless-drive
        \end{center}
        \vspace{-.1cm}
    \end{minipage}
    \begin{minipage}{.48\textwidth}
        \includegraphics[width=\textwidth]{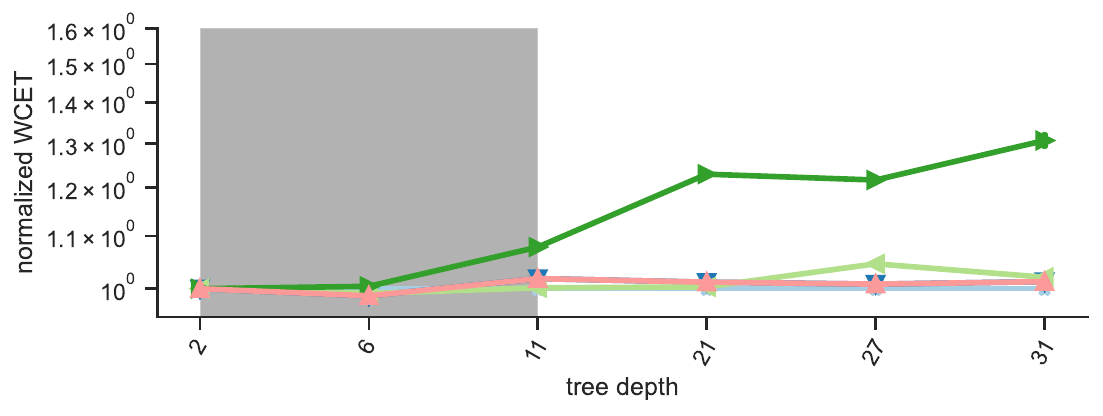}
        \vspace{-.8cm}
        \begin{center}
            spambase
        \end{center}
        \vspace{-.1cm}
    \end{minipage}
    \begin{minipage}{.48\textwidth}
        \includegraphics[width=\textwidth]{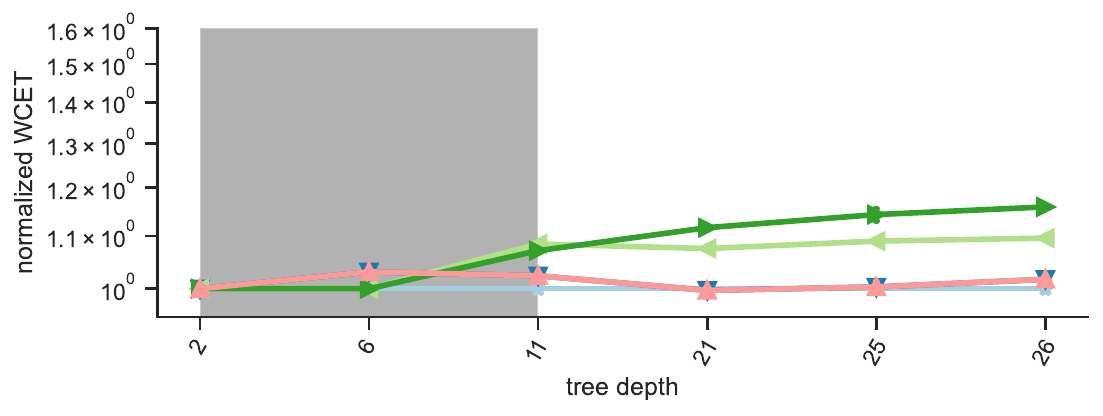}
        \vspace{-.8cm}
        \begin{center}
            wine-quality
        \end{center}
        \vspace{-.1cm}
    \end{minipage}
    \begin{minipage}{.6\textwidth}
        \includegraphics[width=\textwidth]{plots/wcet/Legend.pdf}
    \end{minipage}
    \caption{Normalized analyzed WCET results (based on LLVMTA) for all data sets.
    }
    \label{fig:wcet}
\end{figure*}
\input{plots/table.tex}

\Cref{fig:wcet} illustrates the resulting analyzed WCET by LLVMTA for the considered implementations.
We again normalize the data to the results from our optimization algorithm; thus,
a data point above $1$ indicates
a worse analyzed WCET and a data point below~$1$ indicates
a better analyzed WCET.
The overall trends for the surrogate model and the analyzed WCET results are similar. This suggests that the surrogate model serves as an adequate estimate for the analyzed WCET for
optimization purposes.
As before, the behaviour varies for smaller trees and tends to show a stabilized behavior for larger trees.

Our optimization no longer is a lower bound, since the analyzed WCET differs from the surrogate model.
In the case of covertype depth~11, adult depth~6, mnist depth~11 and spambase depth~6,
the ACET optimization from Chen et~al.~\cite{chen2022efficient} outperforms our proposed optimization.
Our implementation as well as the ACET optimization perform significantly better than the Standard implementation in most cases for deep trees. Furthermore, in most cases, we achieve a higher reduction of the analyzed WCET, than the algorithm from Chen et al.~\cite{chen2022efficient}. We further note that the inverted optimization, which serves as an upper bound in the surrogate model, also provides higher analyzed WCET for deep trees.

Towards quantization, \Cref{improvenumbers} shows a condensed form of the presented data, where the geometric mean of the quotient between the Standard implementation and the SurrogateOpt implementation is shown in the upper half , and the quotient between the Chen et~al. Path implementation and the SurrogateOpt implementation is shown in the lower half. Further, the data is provided for deep trees (deeper than 20) and for all considered trees.
It has to be noted here again, that the result for all trees also include the results for trees with an inadequate shape for the data set.
A number larger than 1 in this table indicates a higher WCET of the Standard, respectively the Chen et~al. Path version than the SurrogateOpt version.
Compared to the Standard implementation,
our optimization can improve the estimated surrogate model WCET by up to $20\%$ and by up to $15-17\%$ for the analyzed WCET for deep trees. This suggests that WCET optimization for decision trees is in general important,
since a significant reduction in WCET can have a high influence on the utilization of a system. Comparing our optimization algorithm to the ACET-based optimization from Chen et~al.~\cite{chen2022efficient}, the improvements are around $3-6\%$
in the surrogate model and around $1-4\%$ for the analyzed WCET. This suggests
that ACET optimization tends to also improve the WCET of decision trees as well. However, it also shows that systematically designing a specific WCET optimization algorithm, as being done in this paper,
can achieve further optimized decision trees.

\subsection{Discussion}

From the previously shown results, the benefits from WCET optimization for decision trees can be clearly observed. Without changing the model logic of the machine learning model, the WCET guarantees are significantly improved.
Potentially unexpectedly, ACET optimization
also improves the WCET. One possible reason is that decision trees tend to grow deeper when more data samples are present for a certain subtree. More data samples translate to a higher access probability in terms of ACET. Consequently, when ACET optimization targets highly probable substructures of a decision tree, it likely targets structures
relevant to the WCEP and the resulting WCET.
Further, ACET optimization tries
to reduce the number of cache misses. When this reduction works reliably, it is also beneficial for the WCET.

While prior ACET-focussed work inadvertently also improves the WCET, our work is the first to systematically study this problem.
The evaluation shows that the surrogate model underlying our optimization succinctly captures key aspects influencing the WCET.
Thus, optimizations w.r.t. the surrogate model are also effective in optimizing the analyzed WCET, beyond prior ACET-focussed work.
On the other hand, it is apparent that there is still a small gap between the surrogate model and the analyzed WCET, indicating that modest improvements are still possible if this gap can be closed by an improved surrogate model, a challenge we would like to tackle in future work.


%% file: plots/table.tex
\begin{table*}
    \centering
    \begin{tabular}{ll|c|c|c|c|c|c|c|c|c|c}
        &&\rotatebox{90}{adult}&\rotatebox{90}{bank}&\rotatebox{90}{covertype}&\rotatebox{90}{letter}&\rotatebox{90}{magic}&\rotatebox{90}{mnist}&\rotatebox{90}{satlog}&\rotatebox{90}{sensorless-drive}&\rotatebox{90}{spambase}&\rotatebox{90}{wine-quality}\\
        \hline
        \textbf{Standard} $/$&
        \textbf{Surrogate Model Geomean ($d\geq20$)}
        &
        1.11
        &
        1.13
        &
        1.13
        &
        1.20
        &
        1.17
        &
        1.10
        &
        1.20
        &
        1.19
        &
        1.04
        &
        1.19
        \\
        {SurrogateOpt}&
        \textcolor{gray}{\textbf{Surrogate Model Geomean}}
        &
        \textcolor{gray}{1.08}
        &
        \textcolor{gray}{1.08}
        &
        \textcolor{gray}{1.08}
        &
        \textcolor{gray}{1.13}
        &
        \textcolor{gray}{1.11}
        &
        \textcolor{gray}{1.06}
        &
        \textcolor{gray}{1.14}
        &
        \textcolor{gray}{1.12}
        &
        \textcolor{gray}{1.04}
        &
        \textcolor{gray}{1.12}
        \\
        &\textbf{Analyzed WCET Geomean ($d\geq20$)}
        &
        1.08
        &
        1.08
        &
        1.07
        &
        1.15
        &
        1.15
        &
        1.06
        &
        1.17
        &
        1.11
        &
        1.02
        &
        1.09
        \\
        &\textcolor{gray}{\textbf{Analyzed WCET Geomean}}
        &
        \textcolor{gray}{1.05}
        &
        \textcolor{gray}{1.04}
        &
        \textcolor{gray}{1.05}
        &
        \textcolor{gray}{1.09}
        &
        \textcolor{gray}{1.11}
        &
        \textcolor{gray}{1.04}
        &
        \textcolor{gray}{1.13}
        &
        \textcolor{gray}{1.07}
        &
        \textcolor{gray}{1.01}
        &
        \textcolor{gray}{1.06}
        \\
        \hline
        \textbf{Chen et~al. Path} $/$&
        \textbf{Surrogate Model Geomean ($d\geq20$)}
        &
        1.01
        &
        1.03
        &
        1.09
        &
        1.02
        &
        1.03
        &
        1.03
        &
        1.06
        &
        1.06
        &
        1.02
        &
        1.04
        \\
        {SurrogateOpt}&
        \textcolor{gray}{\textbf{Surrogate Model Geomean}}
        &
        \textcolor{gray}{1.02}
        &
        \textcolor{gray}{1.02}
        &
        \textcolor{gray}{1.06}
        &
        \textcolor{gray}{1.03}
        &
        \textcolor{gray}{1.03}
        &
        \textcolor{gray}{1.02}
        &
        \textcolor{gray}{1.05}
        &
        \textcolor{gray}{1.05}
        &
        \textcolor{gray}{1.03 }
        &
        \textcolor{gray}{1.02}
        \\
        &\textbf{Analyzed WCET Geomean ($d\geq20$)}
        &
        1.01
        &
        1.03
        &
        1.03
        &
        1.02
        &
        1.01
        &
        1.03
        &
        1.06
        &
        1.04
        &
        1.01
        &
        1.01
        \\
        &\textcolor{gray}{\textbf{Analyzed WCET Geomean}}
        &
        \textcolor{gray}{1.01}
        &
        \textcolor{gray}{1.03}
        &
        \textcolor{gray}{1.02}
        &
        \textcolor{gray}{1.02}
        &
        \textcolor{gray}{1.03}
        &
        \textcolor{gray}{1.01}
        &
        \textcolor{gray}{1.04}
        &
        \textcolor{gray}{1.02}
        &
        \textcolor{gray}{1.01}
        &
        \textcolor{gray}{1.01}
        \\
        \hline
    \end{tabular}
    \caption{Estimated and Analyzed WCET
      for the Standard and Chen et~al. Path implementation over SurrogateOpt.
        }
    \label{improvenumbers}
\end{table*}

%% file: sections/related.tex
Although not specifically targeting decision trees, WCET optimization has been frequently discussed in the literature. 
Previous work for WCET optimization focuses on optimizing the memory layout \cite{WCET_Cache, WCET_SPM, WCET_Code}, register allocation \cite{WCET_REG}, resource allocation \cite{WCET_Resource}, loop unrolling \cite{WCET_LOOP}, and applying compiler optimizations along the worst-case execution path \cite{WCEP_OPT}, short WCEP\@.
Many of these optimizations are implemented in WCET-aware compilers such as WCC \cite{WCC} and Patmos \cite{PATMOS}.
The Patmos project not only provides a WCET compiler, but also contributed a co-design of an WCET-analyzable core and its toolchain.
Plazar et~al. take the impact of branches to WCET specifically into account by developing a code positioning method, which optimizes for the effects of branch predictors \cite{plazar2011wcet}.

Implementations of decisions trees result in an unusual large number of basic blocks compared to other software, which quickly makes these compiler-based approaches infeasible.
While compilers take the program's source code as input 
to generate assembly, we use a high-level model of a decision tree
as input to generate C code.

To the best of our knowledge, other methods, which take a high-level model of the program into account for WCET optimization, do not exist.
However, many approaches to speed up the ACET of a program by taking a higher-level model into account do exist.
Some of these approaches rely on the MLIR \cite{MLIR} to keep a higher-level model of the program and then use that to improve execution time \cite{google, 9923840}.
Alternatively, the ACETONE \cite{ACETONE} approach
generates WCET-analyzable C code of deep neural networks (DNNs) for single-core CPUs.

Apart from WCET-aware optimization, many hardware-aware approaches specific to decision trees and/or random forests for optimizing the ACET exist.
Asadi et~al.\ \cite{Asadi/etal/2014} introduce the concept of hardware-near code generation for decision trees, where the basic structure of data-based \emph{native-trees} and code-based \emph{if-else-trees} is introduced.
Buschjäger et~al.\ \cite{8594826} enrich the optimization of random forests with a probabilistic model of the average-case execution, where the behavior of instruction and data caches is optimized with respect to this model.
Cho and Li propose \textit{Treelite}~\cite{Cho2018}, where the concept of flipping conditions considered in this work is realized through the compiler intrinsic \texttt{\_\_builtin\_expect()} to provide the branch prediction information.
Chen et~al.\ \cite{chen2022efficient} improve these methods by including precise estimation of the node sizes on various execution targets. Hakert et~al.\ \cite{hakert2023immediate} further reduce the data-cache misses of if-else-tree implementations when using floating-point values.


%% file: sections/conclusion.tex
In this paper, we present the first systematic study of WCET optimization for decision trees with if-else-tree implementations.
We propose an approach that leverages a linear compositional surrogate model that estimates path WCETs to algorithmically optimize the WCET of a given decision tree and the corresponding implementation as an if-else-tree.
The surrogate model provides a relation between the depth of a path through the tree and the number of taken branches on that path to the resulting WCET\@.
We propose an optimization algorithm, which takes the corresponding model parameters into account and provides an implementation of an if-else-tree that minimizes the WCET with respect to the surrogate model.
To study in which scenarios the surrogate model is accurate, we provide a detailed analysis of the surrogate model itself and the relation to real-world decision trees.
We compare our optimization algorithm, using a surrogate model, in an experimental evaluation to state-of-the-art heuristic ACET decision-tree performance optimization.

Both the discussion for the surrogate models correlation and
for the WCET reduction in the evaluation show that our proposed algorithm can reduce the analyzed WCET of most decision trees beyond the state-of-the-art optimization.


%% file: sections/outlook.tex
We plan to consider potential influences on the WCET, which cannot be optimized in the source code of C implementations, and integrate them into the systematic WCET reduction in the future.
We also plan to investigate the effect of other WCET drivers apart from branches and cache misses.

%% file: arxiv.bbl
\begin{thebibliography}{10}

\bibitem{Asadi/etal/2014}
N.~Asadi, J.~Lin, and A.~P. de~Vries.
\newblock Runtime optimizations for tree-based machine learning models.
\newblock {\em IEEE Transactions on Knowledge and Data Engineering},
  26(9):2281--2292, Sept 2014.

\bibitem{asuncion2007uci}
Arthur Asuncion and David Newman.
\newblock {UCI} machine learning repository, 2007.

\bibitem{sensing}
Mariana Belgiu and Lucian Dr{\u a}gu{\c t}.
\newblock Random forest in remote sensing: A review of applications and future
  directions.
\newblock {\em ISPRS journal of photogrammetry and remote sensing}, 114:24--31,
  April 2016.

\bibitem{DBLP:conf/pkdd/BockermannBBEMR15}
Christian Bockermann, Kai Br{\"{u}}gge, Jens Bu{\ss}, Alexey Egorov, Katharina
  Morik, Wolfgang Rhode, and Tim Ruhe.
\newblock Online analysis of high-volume data streams in astroparticle physics.
\newblock In Albert Bifet, Michael May, Bianca Zadrozny, Ricard Gavald{\`{a}},
  Dino Pedreschi, Francesco Bonchi, Jaime~S. Cardoso, and Myra Spiliopoulou,
  editors, {\em Machine Learning and Knowledge Discovery in Databases -
  European Conference, {ECML} {PKDD} 2015, Porto, Portugal, September 7-11,
  2015, Proceedings, Part {III}}, volume 9286, pages 100--115, 2015.

\bibitem{breiman1984classification}
L.~Breiman, J.~Friedman, C.J. Stone, and R.A. Olshen.
\newblock {\em Classification and Regression Trees}.
\newblock Taylor \& Francis, 1984.

\bibitem{8594826}
Sebastian Buschj\"ager, Kuan-Hsun Chen, Jian-Jia Chen, and Katharina Morik.
\newblock Realization of random forest for real-time evaluation through tree
  framing.
\newblock In {\em 2018 IEEE International Conference on Data Mining (ICDM)},
  2018.

\bibitem{chen2022efficient}
Kuan-Hsun Chen, ChiaHui Su, Christian Hakert, Sebastian Buschj{\"a}ger,
  Chao-Lin Lee, Jenq-Kuen Lee, Katharina Morik, and Jian-Jia Chen.
\newblock Efficient realization of decision trees for real-time inference.
\newblock {\em Transactions on Embedded Computing Systems}, 2022.

\bibitem{Cho2018}
Hyunsu Cho and Mu~Li.
\newblock Treelite: Toolbox for decision tree deployment.
\newblock In {\em Proceedings of the SysML Conference}, 2018.

\bibitem{WCETO3}
Micka{\"e}l Dardaillon, Stefanos Skalistis, Isabelle Puaut, and Steven Derrien.
\newblock Reconciling compiler optimizations and {WCET} estimation using
  iterative compilation.
\newblock In {\em {RTSS 2019 - 40th IEEE Real-Time Systems Symposium}}, pages
  1--13, Hong Kong, China, December 2019. {IEEE}.

\bibitem{ACETONE}
Iryna De~Albuquerque~Silva, Thomas Carle, Adrien Gauffriau, and Claire Pagetti.
\newblock {ACETONE: Predictable Programming Framework for ML Applications in
  Safety-Critical Systems}.
\newblock In Martina Maggio, editor, {\em 34th Euromicro Conference on
  Real-Time Systems (ECRTS 2022)}, volume 231 of {\em Leibniz International
  Proceedings in Informatics (LIPIcs)}, pages 3:1--3:19, Dagstuhl, Germany,
  2022. Schloss Dagstuhl -- Leibniz-Zentrum f{\"u}r Informatik.

\bibitem{WCET_SPM}
Heiko Falk and Jan~C. Kleinsorge.
\newblock Optimal static {WCET}-aware scratchpad allocation of program code.
\newblock In {\em Proceedings of the 46th Annual Design Automation Conference},
  DAC '09, page 732–737, New York, NY, USA, 2009. Association for Computing
  Machinery.

\bibitem{WCC}
Heiko Falk, Paul Lokuciejewski, and Henrik Theiling.
\newblock Design of a {WCET}-aware {C} compiler.
\newblock In {\em 2006 IEEE/ACM/IFIP Workshop on Embedded Systems for Real Time
  Multimedia}, pages 121--126, 2006.

\bibitem{WCET_REG}
Heiko Falk, Norman Schmitz, and Florian Schmoll.
\newblock {WCET}-aware register allocation based on integer-linear programming.
\newblock In {\em 2011 23rd Euromicro Conference on Real-Time Systems}, pages
  13--22, 2011.

\bibitem{DBLP:journals/ijcv/FanelliDGFG13}
Gabriele Fanelli, Matthias Dantone, Juergen Gall, Andrea Fossati, and Luc~Van
  Gool.
\newblock Random forests for real time 3d face analysis.
\newblock {\em Int. J. Comput. Vis.}, 101(3):437--458, 2013.

\bibitem{llvmta}
Sebastian Hahn, Michael Jacobs, Nils H\"{o}lscher, Kuan-Hsun Chen, Jian-Jia
  Chen, and Jan Reineke.
\newblock {LLVMTA: An LLVM-Based WCET Analysis Tool}.
\newblock In Cl\'{e}ment Ballabriga, editor, {\em 20th International Workshop
  on Worst-Case Execution Time Analysis (WCET 2022)}, volume 103 of {\em Open
  Access Series in Informatics (OASIcs)}, pages 2:1--2:17, Dagstuhl, Germany,
  2022. Schloss Dagstuhl -- Leibniz-Zentrum f{\"u}r Informatik.

\bibitem{SIC}
Sebastian Hahn and Jan Reineke.
\newblock Design and analysis of {SIC}: A provably timing-predictable pipelined
  processor core.
\newblock {\em Real-Time Systems}, 56(2):207--245, 2020.

\bibitem{Hahn2015}
Sebastian Hahn, Jan Reineke, and Reinhard Wilhelm.
\newblock {\em Toward Compact Abstractions for Processor Pipelines}, pages
  205--220.
\newblock Springer International Publishing, Cham, 2015.

\bibitem{hakert2023immediate}
Christian Hakert, Kuan-Hsun Chen, and Jian-Jia Chen.
\newblock Immediate split trees: Immediate encoding of floating point split
  values in random forests.
\newblock In {\em Machine Learning and Knowledge Discovery in Databases:
  European Conference, ECML PKDD 2022, Grenoble, France, September 19--23,
  2022, Proceedings, Part V}, pages 531--546. Springer, 2023.

\bibitem{PATMOS}
Stefan Hepp, Benedikt Huber, Jens Knoop, Daniel Prokesch, and Peter~P Puschner.
\newblock The platin tool kit-the {T-CREST} approach for compiler and {WCET}
  integration.
\newblock In {\em Proceedings 18th Kolloquium Programmiersprachen und
  Grundlagen der Programmierung, KPS}, pages 5--7, 2015.

\bibitem{665905b2-6123-3642-832e-05dbc1f48979}
M.~G. Kendall.
\newblock A new measure of rank correlation.
\newblock {\em Biometrika}, 30(1/2):81--93, 1938.

\bibitem{google}
Rasmus~Munk Larsen and Tatiana Shpeisman.
\newblock Tensorflow graph optimizations, 2019.

\bibitem{MLIR}
Chris Lattner, Mehdi Amini, Uday Bondhugula, Albert Cohen, Andy Davis, Jacques
  Pienaar, River Riddle, Tatiana Shpeisman, Nicolas Vasilache, and Oleksandr
  Zinenko.
\newblock {MLIR}: Scaling compiler infrastructure for domain specific
  computation.
\newblock In {\em 2021 IEEE/ACM International Symposium on Code Generation and
  Optimization (CGO)}, pages 2--14. IEEE, 2021.

\bibitem{WCET_Cache}
Paul Lokuciejewski, Heiko Falk, and Peter Marwedel.
\newblock {WCET}-driven cache-based procedure positioning optimizations.
\newblock In {\em 2008 Euromicro Conference on Real-Time Systems}, pages
  321--330, 2008.

\bibitem{WCET_LOOP}
Paul Lokuciejewski and Peter Marwedel.
\newblock Combining worst-case timing models, loop unrolling, and static loop
  analysis for {WCET} minimization.
\newblock In {\em 2009 21st Euromicro Conference on Real-Time Systems}, pages
  35--44, 2009.

\bibitem{DBLP:journals/natmi/LundbergECDPNKH20}
Scott~M. Lundberg, Gabriel~G. Erion, Hugh Chen, Alex~J. DeGrave, Jordan~M.
  Prutkin, Bala Nair, Ronit Katz, Jonathan Himmelfarb, Nisha Bansal, and
  Su{-}In Lee.
\newblock From local explanations to global understanding with explainable {AI}
  for trees.
\newblock {\em Nat. Mach. Intell.}, 2(1):56--67, 2020.

\bibitem{6751433}
Javier Marín, David Vázquez, Antonio~M. López, Jaume Amores, and Bastian
  Leibe.
\newblock Random forests of local experts for pedestrian detection.
\newblock In {\em 2013 IEEE International Conference on Computer Vision}, pages
  2592--2599, 2013.

\bibitem{plazar2011wcet}
Sascha Plazar, Jan Kleinsorge, Heiko Falk, and Peter Marwedel.
\newblock Wcet-driven branch prediction aware code positioning.
\newblock In {\em Proceedings of the 14th international conference on
  Compilers, architectures and synthesis for embedded systems}, pages 165--174,
  2011.

\bibitem{9923840}
Ashwin Prasad, Sampath Rajendra, Kaushik Rajan, R~Govindarajan, and Uday
  Bondhugula.
\newblock Treebeard: An optimizing compiler for decision tree based {ML}
  inference.
\newblock In {\em 2022 55th IEEE/ACM International Symposium on
  Microarchitecture (MICRO)}, pages 494--511, 2022.

\bibitem{Reineke07}
Jan Reineke, Daniel Grund, Christoph Berg, and Reinhard Wilhelm.
\newblock Timing predictability of cache replacement policies.
\newblock {\em Real Time Syst.}, 37(2):99--122, 2007.

\bibitem{9925686}
Enrico Tabanelli, Giuseppe Tagliavini, and Luca Benini.
\newblock Optimizing random forest-based inference on {RISC-V} {MCUs} at the
  extreme edge.
\newblock {\em IEEE Transactions on Computer-Aided Design of Integrated
  Circuits and Systems}, 41(11):4516--4526, 2022.

\bibitem{WCET_Resource}
Man-Ki Yoon, Jung-Eun Kim, and Lui Sha.
\newblock Optimizing tunable {WCET} with shared resource allocation and
  arbitration in hard real-time multicore systems.
\newblock In {\em 2011 IEEE 32nd Real-Time Systems Symposium}, pages 227--238,
  2011.

\bibitem{WCEP_OPT}
W.~Zhao, W.~Kreahling, D.~Whalley, C.~Healy, and F.~Mueller.
\newblock Improving {WCET} by optimizing worst-case paths.
\newblock In {\em 11th IEEE Real Time and Embedded Technology and Applications
  Symposium}, pages 138--147, 2005.

\bibitem{WCET_Code}
Wankang Zhao, David Whalley, Christopher Healy, and Frank Mueller.
\newblock Improving {WCET} by applying a {WC} code-positioning optimization.
\newblock {\em ACM Trans. Archit. Code Optim.}, 2(4):335–365, December 2005.

\bibitem{10.5555/3172077.3172386}
Zhi-Hua Zhou and Ji~Feng.
\newblock Deep forest: towards an alternative to deep neural networks.
\newblock In {\em International Joint Conference on Artificial Intelligence},
  page 3553–3559, 2017.

\end{thebibliography}
